\documentclass[11pt,oneside,letter]{article}
\usepackage[top=1 in, bottom=1 in, left=0.85 in, right=0.85 in]{geometry}

\usepackage{amsmath,amsfonts,bm}

\def\eqref#1{equation~\ref{#1}}

\def\1{\bm{1}}

\def\ra{{\textnormal{a}}}

\def\rv{{\textnormal{v}}}

\def\rvx{{\mathbf{x}}}

\def\ervc{{\textnormal{c}}}

\def\ervx{{\textnormal{x}}}

\def\vmu{{\bm{\mu}}}

\DeclareMathAlphabet{\mathsfit}{\encodingdefault}{\sfdefault}{m}{sl}
\SetMathAlphabet{\mathsfit}{bold}{\encodingdefault}{\sfdefault}{bx}{n}

\def\gA{{\mathcal{A}}}

\def\gE{{\mathcal{E}}}

\def\gN{{\mathcal{N}}}

\def\gS{{\mathcal{S}}}

\def\sH{{\mathbb{H}}}

\def\sP{{\mathbb{P}}}

\def\sR{{\mathbb{R}}}

\newcommand{\E}{\mathbb{E}}

\DeclareMathOperator*{\argmax}{arg\,max}
\DeclareMathOperator*{\argmin}{arg\,min}

\usepackage{float}
\usepackage{times}
\usepackage{mathtools}
\usepackage[round]{natbib}
\usepackage{makecell}
\usepackage{amsmath,amssymb,graphicx,url}
\usepackage{thmtools,thm-restate,wrapfig,enumitem,mathabx}
\usepackage{booktabs}
\usepackage{color}
\usepackage{subfigure}
\usepackage{colortbl}
\usepackage{multirow}

\usepackage[utf8]{inputenc} 
\usepackage[T1]{fontenc}    
\usepackage{hyperref}       
\usepackage{url}            
\usepackage{booktabs}       
\usepackage{amsfonts}       
\usepackage{nicefrac}       
\usepackage{microtype}      
\usepackage{xcolor}         

\usepackage{algorithm}
\usepackage[noend]{algorithmic}
\usepackage{titletoc}
\usepackage{bbm}

\newtheorem{theorem}{Theorem}

\newtheorem{lemma}{Lemma}

\newtheorem{proof}{Proof}
\newtheorem{definition}{Definition}
\newtheorem{proposition}{Proposition}

\newcommand{\prob}{\mathbb{P}}

\newcommand{\vnu}{\bm{\nu}}
\newcommand{\vsigma}{\bm{\sigma}}

\title{Multi-Metric Adaptive Experimental Design Under a Fixed Budget with Validation}
\date{}
\author{
  Qining Zhang\thanks{The research was mainly conducted when Qining Zhang was a research scientist intern at Amazon.}\\
  University of Michigan\\
  \texttt{qiningz@umich.edu} \\
  \and
  Tanner Fiez \\
  Amazon \\
  \texttt{tannfiez@amazon.com} \\
  \and
  Yi Liu \\
  Amazon \\
  \texttt{yiam@amazon.com} \\
  \and
  Wenyang Liu \\
  Amazon \\
  \texttt{lwenyang@amazon.com} \\
}

\begin{document}
\maketitle

\begin{abstract}
  A/B tests in online experiments face statistical power challenges when testing multiple candidates simultaneously, while adaptive experimental designs (\texttt{AED}) alone fall short in inferring experiment statistics such as the average treatment effect, especially with many metrics (e.g., revenue, safety) and heterogeneous variances. This paper proposes a fixed-budget multi-metric \texttt{AED} framework with a two-phase structure: an adaptive exploration phase to identify the best treatment, and a validation phase with an A/B test to verify the treatment's quality and infer statistics. We propose \texttt{SHRVar}, which generalizes sequential halving (\texttt{SH}) with a novel relative-variance-based sampling and an elimination strategy built on reward $z$ values. It achieves a provable error probability that decreases exponentially, where the exponent $H_3$ generalizes the complexity measure for \texttt{SH} and \texttt{SHVar} with homogeneous and heterogeneous variances, respectively. Numerical experiments demonstrate its performance and robustness.
\end{abstract}

\section{Introduction}
Randomized online experimental design, which aims to evaluate and compare the performance of different system versions, is a standard procedure to support statistically valid decisions in industrial developments, such as new webpage layouts, new services, and new software features. A/B test (or A/B/N test), where the designer assigns an experiment subject to either the control (current version) or one of the multiple treatments (new versions) with a fixed probability and then measures their response, has demonstrated empirical success~\citep{kohavi2020abtest} due to its easy-to-implement nature and the accuracy in inferring experimental statistics such as the average treatment effect (\texttt{ATE})~\citep{Imbens04ATE}. However, A/B tests demonstrate poor statistical power when the number of treatments scales, which becomes insufficient in modern experiments where hundreds of new treatments may be developed simultaneously from machine learning methods such as generative AIs. On the other hand, adaptive experimental design (\texttt{AED}) methods, such as best arm identification (\texttt{BAI}), where the probability of experiment subject assignment can be adaptively chosen, have received increasing popularity as an alternative to reduce the cost of experimentation. However, the adaptivity also prohibits designers from the merits of classic A/B tests, such as the accurate inference of \texttt{ATE}~\citep{cook2023ATE,deep2023ATE}.

\textbf{Experimentation Framework.} To combine the advantages of both \texttt{AED} and A/B tests, specifically fast and efficient treatment selection in experimentation with accurate statistical inference, we propose and study an experimentation framework shown in Fig.~\ref{fig:framework}. This framework consists of two phases, an exploration phase and a validation phase. In exploration, an \texttt{AED} method, such as a multi-armed bandit (\texttt{MAB}) algorithm, is employed to identify the ``best'' treatment, for example, a new layout of the shopping website preferred by customers. Then, in validation, the recommended treatment (new layout) and the control (old layout) are placed in an A/B test to verify their performance in multiple metrics and obtain experiment statistics such as the \texttt{ATE}. When the superior-than-control performance of the recommended treatment is validated under all metrics, the new layout will be deployed into production to replace the old website layout, and the cycle continues. Therefore, the exploration phase aims to identify a treatment most likely to pass the validation.

\textbf{Challenges.} Both \texttt{AED}~\citep{jamieson14bai,kaufmann16bai} and A/B tests~\citep{siroker2015ABtest,kaufmann2014ABtest} have been extensively studied, but the new two-phase experimentation framework poses new challenges. First, most real-world industrial experiments consist of multiple metrics such as revenue, safety, and customer engagement. Second, the control plays a different role from all other treatments in exploration, and thus, its experimental subject assignment should also be different. Third, the validation and exploration phases are correlated, and we aim to maximize the probability of successful validation. Both the \texttt{AED} and A/B test literature fail to address these challenges simultaneously. 

\textbf{Related Work.} Our paper is closely related to best arm identification under a fixed budget. Classic \texttt{AED} algorithms in this setting~\citep{karnin13bai, bubeck11bai,jamieson14bai,garivier16bai}, including \texttt{SH} and \texttt{SR}, often focus on a single metric as a reward. While follow-up works such as the feasible arm identification~\citep{katz-samuels19feasible,katz-samuels18feasiblearm,katz-samuels19feasiblebest} or Pareto set identification~\citep{auer16pareto} extend \texttt{AED} to multiple metrics, these works still do not distinguish between control and treatments, where both are modeled as a mere ``arm''. Some work \citep{russac21control} considers an explicit control arm with the multi-population problem, but as with most other existing works, the exploration and validation phases are studied separately. The intertwined influence across phases has not been well understood despite its practical importance. Our framework follows the general variance equalization principle commonly used both in the fixed-budget setting~\citep{lalitha23variance} and the fixed-confidence setting~\citep{weltz23variance}, and generalizes \texttt{SH} into multi-metric problems.

\textbf{Contributions.}
To address these limitations, we formulate a multi-metric \texttt{MAB} problem with both exploration and validation as the framework shown in Fig.~\ref{fig:framework}, where the goal is to design a fixed-budget \texttt{BAI} algorithm to identify the treatment with the best chance of passing the A/B test validation of all metrics:
\begin{itemize}
    \item We formulate the multi-metric \texttt{AED} problem with exploration and validation, and study both Bayesian and non-Bayesian validation to characterize the best treatment with $z$ values. 
    \item We propose a sequential halving algorithm called \texttt{SHRVar} with a novel relative-variance-based sampling and estimates of reward $z$ values for treatment elimination. \texttt{SHRVar} enjoys strong theoretical guarantees, where the error probability decreases exponentially as the exploration budget increases. The exponent depends on $H_3$ complexity, which generalizes results of \texttt{SH}~\citep{karnin13bai} for classic \texttt{MAB} and \texttt{SHVar}~\citep{lalitha23variance} for \texttt{MAB} with heterogeneous variance. 
    \item We conduct numerical experiments to demonstrate the superior performance of \texttt{SHRVar}.
\end{itemize}

\section{Related Works}
Adaptive experimental design, as a form of online sequential decision-making problems, has been studied in various contexts, e.g., regret minimization~\citep{auer02ucb,auer10ucb} and pure exploration~\citep{audibert10bai,garivier16bai} for multi-armed bandits, Bayesian optimization, and active learning. Our work is similar to the best arm identification formulation, where a decision on recommendation is made following an exploration phase. This section reviews existing papers related to our problem formulation.

\subsection{Pure Exploration Objectives}
Most pure exploration problems strive to model a typical product selection pipeline where an exploration phase is conducted to experiment with the quality of multiple objects (arms or actions). After the exploration, one arm (or multiple products) with the desired quality, measured by the reward, is chosen to launch and handed over to future pipelines. For different practical scenarios, the notion of the desired arm(s) varies in the literature, and we review these notions:

\textbf{Best Arm Identification.} The classic best arm identification problem~\citep{audibert10bai,garivier16bai} considers only the best arm as the desired arm for recommendation, i.e., the arm with the highest ground-truth reward among all candidates. The best arm identification problem is majorly divided into two strings according to the optimization criteria: the fixed confidence setting~\citep{jamieson14bai, garivier16bai,zhang2023BAI} where the goal is to minimize the sample complexity (number of rounds) before identifying the best arm with confidence higher than a pre-specified threshold $1-\delta$, and the fixed budget setting~\citep{bubeck11bai, audibert10bai,russo16bai,karnin13bai,abbasi-yadkori18bai} where the goal is to minimize the probability of identification error with a fixed number of interactions. Even though the fixed-confidence problem is better studied, algorithms developed under this setting, such as TAS~\citep{garivier16bai}, are rarely implemented in practice. On the other hand, successive halving~\citep{karnin13bai} and the top two algorithms~\citep{russo16bai, jourdan22toptwo}, which arose from the fixed-budget setting, are more popular in practice.

\textbf{Top Arms Identification.} In some scenarios, the production pipeline requires more than one arm to be recommended, which motivated the top $K$ arms identification problem~\citep{Kalyanakrishnan10topk, Kalyanakrishnan12topk, kaufmann16bai, chen17topk,zhou22topk} where the goal is to identify the set of all $K$ arms with reward ranking from one to $K$. Even though the top $K$ arms identification problem generalizes best arm identification with broadened practical insights, the algorithms developed under this setting usually mimic best arm identification algorithms such as TAS or successive halving.

\textbf{$\boldsymbol{\varepsilon}$-Good Arm Identification.} In the case where the reward of the best arm is very close to the second-best arm, it usually takes a tremendous number of samples to identify the exact best arm, while the quality improvement is small. Therefore, an intuitive bypass is to settle with a good enough arm, i.e., an arm that has a reward close to the best arm, which motivates the $\boldsymbol{\varepsilon}$-Good Arm Identification problem~\citep{mason20epsilongood, katz-samuels20goodarm, zhao23epsilongood}. In this problem, the objective is to settle with an arm (or all arms) whose reward is within the $\varepsilon$ gap to the best arm. A slightly different problem of similar insight is to find an arm (or all arms) with a reward larger than some fraction of the best arm, which is studied under the level set estimation literature with implicit threshold level~\citep{gotovos13levelset}. Both formulations require the knowledge of $\varepsilon$ for algorithm design and typically stop exploration as soon as the ambiguity between arms is less than this threshold.

\textbf{Good Arm Identification.} The good arm identification problem is an even weaker notion than $\boldsymbol{\varepsilon}$-good arm identification, where here the objective is simply to identify one or all arms that have a reward larger than some fixed threshold~\citep{kano19good,degenne19goodarm}, and output them as quickly as possible. In this problem, the exploration and exploitation trade-off lies within the arms that are close to the threshold, i.e., whether to pull an arm that is close to the threshold to reduce confusion or to pull an empirical good arm to increase confidence. Different algorithms have been proposed to solve this problem with anytime characterization~\citep{jourdan23goodarm}, data-driven threshold~\citep{tsai23goodarm}, and small gaps settings~\citep{tsai24goodarm}. \citep{kaufmann18goodarmexist} studied an easier problem to identify the existence of a good arm with a reward larger than some threshold. The good arm identification problem is internally closely related to two similar problems, i.e., thresholding bandits~\citep{locatelli16thresholding, mukherjee17thresholding,ouhamma21thresholding} and level set estimation with explicit level threshold~\citep{bryan05levelset,gotovos13levelset}. In thresholding bandits, the goal is to classify the arms into two sets according to the ground truth reward, one with a reward higher than the threshold and one with a reward lower than the threshold. The simple regret performance of the thresholding bandit is also studied in~\citep{tao19thresholding}. The level set estimation problem also intends to classify the input space (arm space) into two sets, one with a function value (reward) larger than the threshold and the other with a function value lower than the threshold. However, level set estimation problems usually deal with continuous input spaces or discredited versions of them~\citep{shekhar19levelset,ngo24levelset}. Moreover, the function value is usually assumed to follow a Gaussian process prior~\citep{gotovos13levelset}, and the algorithm to solve it will make decisions based on the acquisition function~\citep{zanette19levelset,senadeera20levelset,bogunovic16levelsetbo}, which requires the knowledge of this Gaussian process before computing. The results are also generalized with heterogeneous variance scenarios~\citep{inatsu19levelset,iwazaki20levelset}. Attempts to reduce the Gaussian process prior to approximations have also been made. In particular, \citep{mason22levelset} proposed an RKHS approximation, and \citep{ha21levelset} used neural networks as an approximation. The theoretical limits of level set estimation are characterized in~\citep{bachoc21levelset}. Both problems are deeply connected with the broad active learning literature and have applications in classification tasks~\citep{jain19activelearning}.

\textbf{Better Arm Than Control.} To the best of our knowledge, the only work in the literature that studies recommending an arm with a higher reward than the control arm is~\citep{russac21control}, where the goal is to identify an arm better than the known control arm with reward performance averaged over multiple populations. The sampling strategy distinguishing control and treatment arms is also studied. However, their developed algorithm is inherent to TAS for best arm identification, which requires large computational resources. Moreover, their recommended arm is agnostic to the variance and thus may result in poor validation power. Thirdly, their approach towards multiple metrics is a simple weighted sum, while we aim for the robust worst-case metric.

\subsection{Bandits with Multiple Metrics}
Our work is also closely related to bandits with multiple objectives since we consider multiple reward metrics. In this model, the arm is associated with a joint high-dimensional distribution that generates a vector of rewards upon arm pulls. We also review the literature on multi-objective bandits.

\textbf{Pareto Front Identification.} The notion of the best arm is vague in the context of multi-dimensional reward, and therefore a large body of works study the Pareto optimality framework where the goal is to identify all arms that lie on the Pareto front in terms of the groud-truth reward~\citep{drugan13pareto, auer16pareto, yahyaa14pareto, yahyaa14paretoknowledge,crepon24pareto}. \citep{moffaert14pareto} studies the Pareto front identification problem in continuous action settings, and \citep{drugan14paretoscalar} uses a scalarization approach to identify the Pareto front, which is further generalized to infinite horizon bandits~\citep{drugan15pareto}. Our goal is not to identify the Pareto front. Instead, we identify the one better-than-control arm that has the most testing power against it.

\textbf{Scalarization:} Another approach to defining the best arm with multi-dimensional reward is scalarization with a pre-specified utility function. \citep{roijers17scalerization} defines the best arm as a weighted linear combination of different dimensions of the ground-truth reward vector, and then performs best arm identification under the linear transformation. However, the linear weight vector is usually hard to pre-define in practice, and thus \citep{jeunen24scalarization} proposes a data-driven method to learn the weight vector. Beyond linear scalarization, \citep{ararat23scalarization} proposes to scalarize based on a preference over the dimension, which reduces the problem to a hierarchy of best arm identification problems. A completely data-driven scalarization method is also used in~\citep{zintgraf18scalarization} where the utility function is learned from preference data. Our approach falls into the general category of scalarization to treat multiple metrics. However, different from the above works, the scalarization criterion depends on the ground-truth reward and thus is unknown to us, and we have no access to auxiliary data as in~\citep{jeunen24scalarization} for utility function learning. Another similar line of work is feasible arm identification~\citep{katz-samuels18feasiblearm}, a generalization of good arm identification to multi-dimensional rewards. Here, an arm is considered feasible if the ground-truth reward lies in a known polyhedron, so the scalarization from high dimension to feasibility also depends on the reward of each arm. Follow-up work~\citep{katz-samuels20goodarm} generalizes the problem into the large arm regime, with application to recommendation systems. Our setting is also different from theirs in several ways. First, they don't apply the preference to feasible arms and intend to identify all of them, while we prefer to identify the one that has the most testing power against the constraint set. Second, their constraint set is known, while ours is determined by a control arm with an unknown reward, which is more challenging since the sampling strategy of the control arm should also be considered.

\textbf{Constrained Optimization:} The most widely studied approach to multi-dimensional reward is to view one reward metric as the primary reward metric for optimization while viewing the other reward metrics as constraints that should exhibit quality higher (or lower) than some threshold. Constrained regret minimization has been well-studied in the bandit literature~\citep{pacchiano21constraint, liu21constraintregret}, but on the other hand, constrained best arm identification has received less attention. The earliest best arm identification problem considering constraints during the exploration phase is bandits with knapsacks~\citep{badanidiyuru18knapsacks,ding13knapsacks}, where the exploration is subject to a hard resource budget constraint and therefore requires careful planning. \citep{kanarios24cost} studied a slightly easier problem, relaxing the hard constraint, which intends to find the best arm while minimizing the constraint violation during exploration. 

\textbf{Constrained Fixed-Confidence Bandits:} In the fixed confidence setting, \citep{chen16constraint} studied the best arm identification problem with matroid constraints, and \citep{katz-samuels19feasiblebest} studied the polyhedron constraint. \citep{wang22constraint} studies the constrained best arm identification problem, where, in addition to each arm, a dose-level modeling clinical trial is also required for selection during exploration and is related to the safety constraint. The goal is to find the best safe arm if given the proper dose level. Moreover, the safety constraint is imposed both during exploration and the final arm it identifies. \citep{camilleri22constraintlinear,shang23constraintlinear} studied a generalized linear bandit best arm identification problem with safety constraints on the exploration phase. \citep{carlsson24constraint} studied the problem of identifying the best mixture of arms in the presence of constraints, instead of a single arm, since an arm mixture could potentially have a better performance quality. Best arm identification problems with more practical constraints, such as variance constraints~\citep{hou23constraint}, risk constraints~\citep{david18constraint}, fairness constraints~\citep{wu23constraint}, and resource constraints~\citep{li24constraint}, are also studied. Another line of work that presents a similar objective is safe Bayesian optimization~\citep{sui15safe,sui18safe}, where the goal is to maximize the function value $f(x)$ satisfying the constraint $g(x)<h$ in the fixed-confidence setting with the presence of Gaussian noise. Similar to level set estimation, most works in this field assume the functions are generated from a known Gaussian process, which provides prior information for the arm identification and enables the construction of an acquisition function~\citep{berkenkamp16safebo,berkenkamp23safebo}. Future studies generalize this field of work to improve the exploration efficiency~\citep{turchetta19safebo}, and to the continuous action~\citep{bottero22safebo} and high-dimensional input settings~\citep{kirschner19safebo,duivenvoorden17safebo}. The performance of safe Bayesian optimization algorithms is also validated in empirical studies~\citep{gelbart14safebo,hernandez-lobato16safebo,gardner14safebo}. \citep{bogunovic16levelsetbo} studied the similarity between Bayesian optimization and level set estimation, and proposed a unified framework. 

\textbf{Constrained Fixed-Budget Bandits.} The fixed-budget setting is less studied compared to the fixed-confidence counterpart. The earliest paper that studies constrained best arm identification in this setting is~\citep{chang20safebai} using an epsilon-greedy algorithm for exploration, which is later shown to be inefficient. An adaptation of successive rejection is proposed in~\citep{faizal22safebai}, which finds the best single arm under constraint, and \citep{tang24safebai} generalized the algorithm to tackle the identification of support of best arm mixtures. \citep{lindner22safebai} studied a slightly different problem where they assume a known reward but an unknown constraint, and the goal is to learn the constraint while identifying the best safe arm. Our work is different from constrained optimization since there is no primary reward metric.

\subsection{Best Arm Identification Formulations}
Considering the large amount of literature studying online sequential learning with different objectives and formulations, we review the works that present similar goals to our paper.

\textbf{Min-Max Optimization:} The minimax optimization was studied largely in the robust optimization literature~\citep{bertsimas10robustoptim,chen17robustoptim} and the game theory literature~\citep{cai11minimax,nouiehed19minimax}, where a natural adversary agnostic to the agent is assumed to pick the worst environment for the agent to perform optimization. This assumption coincides with the risk-averse nature of many applications. The min-max objective is also studied in many works related to multi-armed bandits with structure. For example, \citep{wang22minimax} assumed the arms belonged to certain groups and intended to maximize the worst arm performance in a group. \citep{garivier16minimax, marchesi19minimax} assumed the reward of each arm would also depend on the action of an adversary and thus formulated a min-max game bandit problem. The objective is to pick the best arm robust to the worst selection of the adversary. Beyond bandit literature, the min-max objective is also commonly studied in multi-class classifications~\citep{nouiehed19minimax} and other machine learning settings.

\textbf{Heterogeneous Variance:} One characterization that makes our problem interesting is the heterogeneous variance nature of arms and reward metrics, so we review the best arm identification with fixed-budget literature for works dealing with heterogeneous variance. \citep{lalitha23variance} generalized the sequential halving algorithm to the heterogeneous variance setting and proposed a variance-aware sampling strategy. If the variance is unknown to the practitioner, they also proposed an unbiased estimation with theoretical performance guarantees. \citep{kveton22bayesianfixbudget} studied the Bayesian bandit setting with heterogeneous variance to understand the relationship between performance and prior. Another line of work that typically studies the influence of variance is multi-fidelity bandits~\citep{poiani22fidelity, poiani24fidelity}, motivated by the simulation experiments with heterogeneous qualities, where the goal is to identify the experimental result while controlling the cost incurred by using high-fidelity simulators. The heterogeneous variance effect is also studied in generalized adaptive experimental design~\citep{weltz23variance,lalitha23variance}, and other active learning scenarios~\citep{antos10variance, chaudhuri17variance}.

\section{Preliminaries}\label{sec:preliminary}

In this section, we formulate the \emph{Multi-Metric Multi-Armed Bandit} (\texttt{M3AB}) model shown in Fig.~\ref{fig:framework}, define the best treatment, and then introduce the objective. Throughout the paper, we use $[\cdot]_+$ to denote $\max\{0, \cdot\}$, use $\varpropto$ to refer to ``proportional to'', and use $\lesssim$ to hide absolute constants in $\leq$. For a random variable $\rvx$, we use $\mathrm{Var}(\rvx)$ to denote its variance.

\textbf{\texttt{M3AB} with Heterogeneous Variance.} We consider a stochastic \texttt{MAB} model with $A+1$ arms denoted as $\gA = \{0, 1, \cdots, A\}$. The arm $0$ is called \emph{control}, and all other arms are called \emph{treatments}, as shown in Fig.~\ref{fig:framework}. The rewards of arms are vectors of dimension $M$. We call each dimension a reward \emph{metric} and use $\vnu_a = \nu_{a,1,}\times\cdots\times\nu_{a,M}$ to denote the joint distribution of all reward metrics for arm $a$, where $\nu_{a,i}$ is a normal distribution for reward metric $i$ and arm $a$. The developed algorithm and theoretical guarantees could be generalized to sub-Gaussian and even heavier-tailed distributions following the development in this paper. For each metric $i$ and each arm $a$, the reward distribution has mean $\mu_{a,i}$ and variance $\sigma_{a,i}^2$, where $\sigma_{a,i}$ can be heterogeneous, both across metrics for a fixed arm $a$ or across different arms. Let $\vmu_a = [\mu_{a,1}, \cdots, \mu_{a,M}]$ and $\vsigma_a^2 = [\sigma_{a,1}^2, \cdots, \sigma_{a,M}^2]$. We assume the expectations $\{\vmu_a\}_{a=1}^A$ are unknown to the agent but $\{\vsigma_a^2\}_{a=1}^A$ are known. The \texttt{M3AB} model is a generalization of classic \texttt{MAB}. If $M=1$, we recover the \texttt{MAB} problem with heterogeneous variance studied in~\citep{lalitha23variance}, and if the variances of different arms are the same, we recover the classic \texttt{MAB} model with homogeneous variance~\citep{karnin13bai}. 

\textbf{Exploration (BAI).} The agent explores the treatments and control in a best-arm-identification framework with a fixed budget $T$. At each time $t\in [T]$, the agent pulls an arm $\ra_t$ and a random reward vector $\rvx_t = [\ervx_{t,1}, \cdots, \ervx_{t,M}]$ is sampled from $\vnu_{\ra_t}$, the joint reward distribution of the pulled arm. After observing the reward vector, the agent moves to the next time step $t+1$, and decides which arm to pull to repeat the process. At time $T$, the agent is expected to recommend (select) a \emph{treatment} $\hat{a}\in \gA \setminus \{0\}$ and proceed to validation. 
Ideally, $\hat{a}$ should be the treatment that has a reward better than the control in all metrics, and at the same time maximizes the probability that a significant positive result occurs if it is placed in an A/B test against control. 

\begin{figure*}[t]
    \centering
    \includegraphics[width=0.98\linewidth]{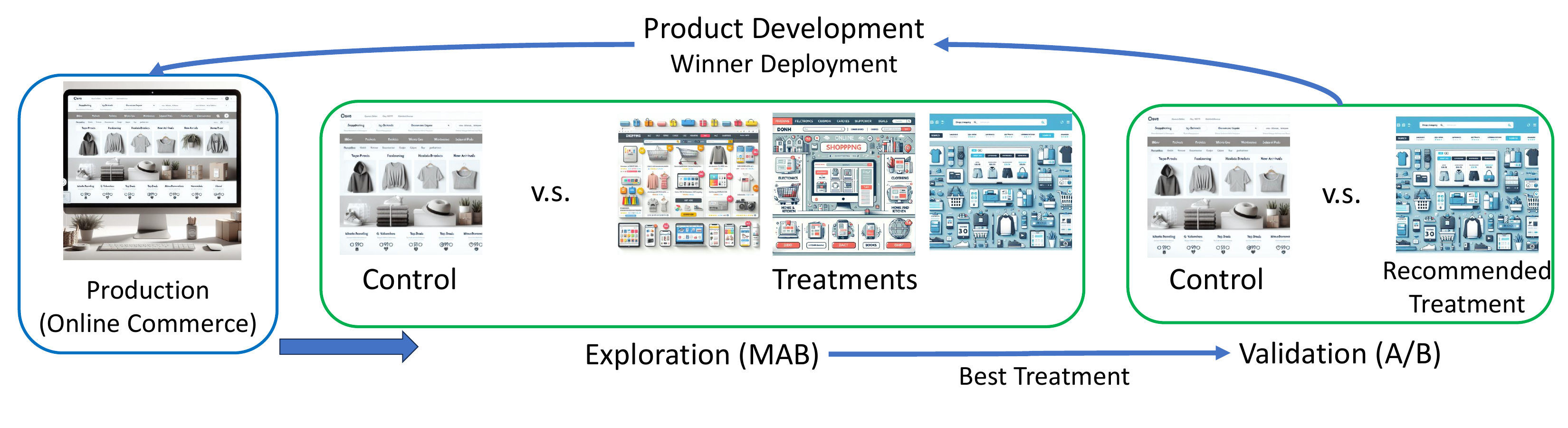}
    \caption{E-Commerce Webpage Development: the framework consists of two phases: exploration and validation. In exploration, a multi-armed bandit algorithm is used to select a treatment. In validation, the selected treatment undergoes an A/B test against the control to validate its quality.}
    \label{fig:framework}
\end{figure*}

\textbf{Validation (A/B).} 
In fixed-budget exploration, the error probability is usually unobserved and can be large. Therefore, a validation A/B test with budget $T_{\rv}$ is conducted to verify the quality of the selected treatment $\hat{a}$ and obtain detailed experiment statistics such as the average treatment effect (ATE).
We consider two types of validation A/B tests, the \emph{non-Bayesian validation} and the \emph{Bayesian validation}, both of which have been widely used in industrial experiments \citep{kirk2009experimental,seltman2012experimental,box2011bayesian}. In both tests, the selected treatment $\hat{a}$ and the control are both pulled $T_{\rv}/2$ times. Then, the empirical means of metric $i$ of both arms are denoted as $\hat{\mu}_{\rv}(\hat{a},i)$ and $\hat{\mu}_{\rv}(0,i)$ respectively. For each metric $i$, we define the null hypothesis and new hypothesis as $\sH_{0,i} = \{\mu_{\hat{a},i} < \mu_{0,i}\}$ and $\sH_{1,i} =\{ \mu_{\hat{a},i} > \mu_{0,i}\}$. If the null hypothesis can be rejected, we say the treatment passes the validation of metric $i$, and denote this event as $\gE_{\rv,i} = \{\text{reject } \sH_{0,i} \text{ in validation}\}$.
For both tests, we measure the probability $\sP(\gE_{\rv,i})$ that the recommended treatment passes validation.

\emph{Non-Bayesain Validation}. The non-Bayesian validation is usually a z-test or a t-test \citep{seltman2012experimental,student1908nonbayesian}. Since the variances $\sigma_{a,i}$ are known, a z-test is conducted between the control arm $0$ and recommended treatment $\hat{a}$ for each metric. Let $\Phi(\cdot)$ be the cumulative distribution function of the standard normal distribution. Given a confidence level $\delta_i$ for each metric $i$, we reject the null hypothesis $\sH_{0,i}$ if:
\begin{align*}
    \hat{\mu}_{\rv}(\hat{a},i) - \hat{\mu}_{\rv}(0,i) \geq \Phi^{-1}(1 - \delta_i) \sqrt{2\left(\sigma_{\hat{a},i}^2 + \sigma_{0,i}^2\right)T_{\rv}^{-1} }.
\end{align*}

\emph{Bayesian Validation} \citep{kamalbasha21bayesian}. We assume the \texttt{ATE} $\mu_{\hat{a}, i} - \mu_{0,i}$ for metric $i$ follows a normal prior distribution $\gN(0, \tau_i^2)$ independent of other metrics, where $\tau_i^2$ is the prior variance. The posterior distribution of \texttt{ATE} for the metric $i$ is normal and calculated as $\nu_{\rv,i} = \gN(\hat{\Delta}_{\rv, i}, \hat{\sigma}^2_{\rv, i})$ \citep{kveton22bayesianfixbudget}, where $\hat{\Delta}_{\rv, i}$ and $\hat{\sigma}^2_{\rv, i}$ are the posterior mean and variances with explicit forms:
\begin{align}
    \hat{\sigma}_{\rv, i}^2 = \left(\frac{ T_{\rm{v}} }{2\left(\sigma_{\hat{a},i}^2 + \sigma_{0,i}^2\right)} + \frac{1}{\tau_i^2} \right)^{-1}, \quad 
    \hat{\Delta}_{\rv, i} = \frac{T_{\rv}\hat{\sigma}_{\rv, i}^2}{2\left( \sigma_{\hat{a},i}^2 + \sigma_{0,i}^2 \right)} \left( \hat{\mu}_{\rv}(\hat{a},i) - \hat{\mu}_{\rv}(0,i)  \right) .
\end{align}
We obtain the posterior $p_i$ of $\{\mu_{\hat{a},i}>\mu_{0,i}\}$ as
$
    p_i = \int_{x = 0}^{\infty} f_{\nu_{\rv,i}}(x) dx = \Phi(\hat{\Delta}_{\rv, i} /\hat{\sigma}^2_{\rv, i}),
$
where $f_{\nu_{\rv,i}}$ is the posterior density. We reject the null hypothesis $\sH_{0,i}$ if $p_i$ is larger than a threshold $q_i$.

\textbf{Robust Treatment Identification.} 
The goal is to find a treatment $\hat{a}$ that is most ``robust'' and ``likely'' to pass the validation for all reward metrics. Specifically, a risk-averse view is considered, and the best treatment $a^*$ is defined as maximizing the minimum probability of rejecting the null hypothesis for all metrics:
\begin{align}
    a^* = \argmax_{a\in \gA / \{0\}} \min_i \sP\left(\gE_{\rv, i}\right). \label{eq:best-prob}
\end{align}
The goal of this paper is to design a BAI algorithm to recommend $a^*$, and the performance is measured by the probability of error, i.e., $\sP(\hat{a} \neq a^*)$.

\subsection{Risk-Averse and Robust Validation}

\begin{table*}[t]
    \centering
    \begin{tabular}{ccccccc}
    \toprule
    \multirow{2}{*}{arms}&  \multicolumn{2}{c}{metric $1$} & \multicolumn{2}{c}{metric $2$} & \multirow{2}{*}{$\sP(\gE_{\rv, 1} \cap \gE_{\rv, 2})$} & \multirow{2}{*}{$\min\{\sP(\gE_{\rv, 1}), \sP(\gE_{\rv, 2})\}$} \\
                    & $(\mu_{a,i}, \sigma_{a,i})$ & $\sP(\gE_{\rv, 1})$ & $(\mu_{a,i}, \sigma_{a,i})$ & $\sP(\gE_{\rv, 2})$ &  &\\
    \midrule
    C         & $(0,10)$   &       & $(0,10)$   &         &       &     \\
    T$1$   & $(0.6,10)$ & 0.44  & $(0.6,10)$ & 0.44    & 0.19  & \textbf{0.44}\\ 
    T$2$   & $(-0.2,30)$ &  0.30 & $(6,10)$   & 0.99    & \textbf{0.30}  & 0.30\\
    \bottomrule
    \end{tabular}
    \caption{\texttt{M3AB} example: $2$ treatments (T1, T2) and $1$ control (C) and a typical Bayesian validation setup in industry with $q_1 = q_2 = 0.67$, prior variance $\tau_1 = \tau_2 = 10$, and $T_{\rv} = 100$. The mean and standard deviations for both metrics are shown, as well as the probability of passing validation for each metric based on the normal posterior. In general, T$1$ is better than control for all metrics, and T$2$ is worse than control in metric $1$. Therefore, T$1$ should be chosen for validation instead of T$2$.}
    \label{tab:example}
\end{table*}

In the real world, a new version should generally not have a poorer quality than its previous version in any metric. Otherwise, it may lead to catastrophic results, e.g., loss of revenue and customer trust. Therefore, the objective of the validation phase shown in Fig.~\ref{fig:framework} is to ensure that the recommended treatment $\hat{a}$ is better in all metrics and avoid any mistakes during exploration. 

Compared to a more natural angle that maximizes $\prob(\cap_i \gE_{\rv, i})$, the probability of rejecting all null hypotheses for all metrics simultaneously in validation, our max-min problem formulation in equation \ref{eq:best-prob} takes a robust perspective that evaluates a treatment not only by its superior quality in some metrics but also the risk of misidentifying its inferiority to control in other metrics. We illustrate with a toy example in Tab.~\ref{tab:example}. This example uses Bayesian validation, and we evaluate the probability of rejecting the null hypothesis for all treatments and metrics based on the expected reward and variance. In this example, only treatment T$1$ has better quality than the control C in both metrics, while T$2$ does not satisfy the aforementioned requirement since it has a worse quality in metric $1$. However, T$2$ has a higher overall validation passing probability $\sP(\gE_{\rv, 1} \cap \gE_{\rv, 2})$ compared to T$1$, i.e., $0.30$ versus $0.19$. This is because the inferiority of T$2$ in metric $1$ is difficult to identify due to the large variance and small effect gap, and the other metric has a large probability of passing $\prob(\gE_{\rv, 2})$ as well, which makes this treatment easy to pass validation in general. Therefore, the exploration phase should avoid recommending such treatments since they violate the development principle and may not be detected through validation. On the other hand, both metrics of T$1$ have a moderate probability of passing so $\sP(\gE_{\rv, 1} \cap \gE_{\rv, 2})$ is smaller compared to T$2$. This shows the joint probability formulation would discourage better-than-control treatments with moderate treatment effects over metrics and mis-detect an inferior treatment if the validations for other metrics are easy to pass. On the other hand, under the max-min problem formulation which aims for the treatment with the largest $\min\{\sP(\gE_{\rv, 1}), \sP(\gE_{\rv, 2})\}$, we will favor T$1$ over T$2$, which aligns with the real-world interests. The equivalence of the two formulations in the large $T_{\rv}$ regime is also discussed in Sec.~\ref{sec:appendix-equiv} of the appendix.

\section{Characterization of the Best Treatment}
Since $a^*$ is defined on the intangible $\sP\left(\gE_{\rv, i}\right)$, we first analyze the validation A/B test, which gives a unified characterization for $a^*$ in both non-Bayesian and Bayesian validations as follows: 
\begin{proposition}\label{prop:val}
    In both non-Bayesian and Bayesian validation tests, the best treatment $a^*$ defined in \eqref{eq:best-prob} is equivalently defined as follows:
    \begin{align}
        a^* = \argmax_{a\in \gA / \{0\}} \min_i \underbrace{\frac{\mu_{a,i} - \mu_{0,i}}{\sqrt{\sigma_{a,i}^2 + \sigma_{0,i}^2}} + \xi_{a,i}}_{z_{a,i}},\label{eq:best-z}
    \end{align}
    where $z_{a,i}$ is the \textbf{z value} and $\xi_{a,i}$ is the \textbf{validation constant} which we obtain from known values as:
    \begin{align*}
        \xi_{a,i} =& \frac{\Phi^{-1}(\delta_i)}{\sqrt{T_{\rv} / 2}} \quad (\text{\rm{non-Bayesian}}), \\
        \xi_{a,i} = & \frac{\Phi^{-1}(1-q_i)}{\sqrt{T_{\rv} / 2}} \sqrt{1 + 2\frac{\sigma_{a,i}^2 + \sigma_{0,i}^2}{\tau_i^2 T_{\rv}}} \quad (\text{\rm{Bayesian}}).
    \end{align*}
\end{proposition}
\textbf{Best Treatment $a^*$.} 
The quality of each treatment is characterized by the $z$ value, which consists of two components: the first is the signal-to-noise ratio, where the numerator represents the \texttt{ATE}, and the variance in the denominator represents the hardness of validation. The second term $\xi_{a,i}$ reflects validation requirements such as the required confidence level and the validation time horizon. The smaller $z_{a,i}$ is, the harder it will be for treatment $a$ to pass validation of metric $i$. Therefore, the best treatment $a^*$ optimizes the minimum $z$-value to ensure a high passing probability for all metrics. For simplicity, we define the bottleneck metric $i_a^*$, where the subscript is omitted when no ambiguity.
\begin{definition}[Bottleneck Metric]
    For each treatment $a$, define the bottleneck metric $i^*_a$ as the metric with the smallest $z$ value, i.e., $i^*_a = \argmin_i z_{a,i}$. 
\end{definition}

\textbf{Variance Heterogeneity.} $a^*$ is influenced by the heterogeneity of variances among both metrics and treatments. Let's consider $M=1$ and omit the subscript $i$. Suppose we use non-Bayesian validation with confidence level $\delta$ and horizon $T_{\rv}$, then all treatments have the same validation constants $\xi_{a} = \Phi^{-1}(\delta) / T_{\rv}$. If furthermore, all treatments have the same variance $\sigma_{a}^2$, then the best treatment $a^*$ would be the arm with the largest reward expectation $\mu_{a}$. However, if treatments have different variances $\sigma_{a}^2$, the arm with the largest reward expectation may no longer be $a^*$ as it may have a large variance, resulting in a small $z$ value. In multi-metric models, the heterogeneity of the validation constant $\xi_{a,i}$ will also influence the definition of $a^*$. For example, in non-Bayesian validations, some metrics are primary while others are considered as guardrails, and therefore, smaller confidence levels $\delta_i$ may be required for primary metrics but larger $\delta_i$ for guardrails. Similarly, in Bayesian validations, the heterogeneous prior variance of different metrics may also lead to heterogeneous validation constants $\xi$ that further influence the choice of the best treatment $a^*$.

\section{A Sequential Halving Algorithm for \texttt{M3AB}}\label{sec:shrvar}
\begin{algorithm}[t]
\caption{Sequential Halving with Relative Variance (\texttt{SHRVar})}\label{alg:mosh}
\begin{algorithmic}[1]
\STATE initialize $\gA_1 \leftarrow \mathcal{A}/\{0\}$;
\FOR{$s=1,2,\cdots, \lceil \log_2 A\rceil$}
    \STATE for control, $N_s(0) =\left \lfloor \frac{ \lambda_{\gA_s, \Sigma} }{ \rho_{\gA_s, \Sigma} + \lambda_{\gA_s, \Sigma}} \frac{T}{\lceil \log_2 A\rceil} \right\rfloor$.
    \STATE for active treatments $a\in \gA_s$, 
    $$N_s(a) = \left \lfloor \frac{\max_i\rho_{a,i}^2}{\rho_{\gA_s, \Sigma} \left( \rho_{\gA_s, \Sigma} + \lambda_{\gA_s, \Sigma} \right) } \frac{T}{\lceil \log_2 A\rceil} \right\rfloor;$$
    \FOR{$a\in \gA_s \cup \{0\}$}
        \STATE sample $a$ for $N_s(a)$ times and obtain reward vectors $\rvx_{a, 1}, \rvx_{a,2}, \cdots, \rvx_{a, N_s(a)}$;
    \ENDFOR
    \STATE compute $\hat{\mu}_s(a, i)$ and $\hat{z}_{s}(a, i)$ as equation~\ref{eq:empirical-mean}.
    \STATE let $\gA_{s+1} \subset \gA_s$ be the set of $\lceil|\gA_s|/2\rceil$ arms with larger $\min_i \hat{z}_s(a,i)$;
\ENDFOR
\STATE recommend the arm in $\mathcal{A}_{\lceil \log_2 A \rceil+1}$
\end{algorithmic}
\end{algorithm}
We propose an algorithm called \emph{Sequential Halving with Relative Variance} (\texttt{SHRVar}) to tackle the proposed \texttt{M3AB} problem. The proposed algorithm is based on the sequential halving framework, and details are summarized in algorithm~\ref{alg:mosh}. In vanilla \texttt{SH}~\citep{karnin13bai}, the total budget $T$ is divided into $\log_2 A$ stages, and all active arms are sampled uniformly. Then, \texttt{SH} obtains the empirical reward of each arm and eliminates half of them with a lower empirical reward. Similarly, in \texttt{SHRVar}, we also divide the time horizon into $\log_2 A$ stages and sample each arm according to a designed fraction (Lines 3-4). Here, we define the \emph{relative variance} $\rho_{a,i}^2$ and $\lambda_{a,i}^2$ which shows the proportion of reward randomness of treatment $a$ compared to control as:
\begin{align}\label{eq:relative-var-def}
    \rho_{a,i}^2 = \frac{\sigma_{a,i}^2}{\sigma_{a,i}^2 + \sigma_{0,i}^2}, \quad \lambda_{a,i}^2 = \frac{\sigma_{0,i}^2}{\sigma_{a,i}^2 + \sigma_{0,i}^2}.
\end{align}
Then, let $\rho^2_{\gA_s, \Sigma} = \sum_{a\in \gA_s}\max_{i}\rho_{a,i}^2$ and $\lambda^2_{\gA_s, \Sigma} = \max_{a \in \gA_s} \max_{i} \lambda^2_{a,i}$ be the overall treatment relative variance in set $\gA_s$ and the overall control relative variance.
Both can be defined for any set $\gS$ of treatments by replacing the subscript. The intuition of the proposed relative-variance-based sampling rule will be discussed shortly. After observing the samples, \texttt{SHRVar} computes the empirical reward $\hat{\mu}_s(a,i)$ and the empirical $z$ values with a plug-in estimator for each treatment and metric of the current stage, i.e., $\forall a\in \gA_s, i \in [M]$,
\begin{subequations}\label{eq:empirical-mean}
\begin{equation}
    \hat{\mu}_s(a,i) = \frac{1}{N_s(a)} \sum_{k=1}^{N_s(a)} \rvx_{a,i, k}, 
\end{equation}
\begin{equation}
    \hat{z}_{s}(a, i) =  \frac{\hat{\mu}_s(a, i) - \hat{\mu}_s(0, i) }{\sqrt{\sigma^2_{a, i} + \sigma^2_{0, i} } } + \xi_{a, i}.
\end{equation}
\end{subequations}
where $\rvx_{a,i, k}$ is the random reward of metric $i$ and treatment $a$ for the $k$-th pull.
Take a minimum over metrics to estimate the $z$ value for the bottleneck metric $\min_i z_{a,i}$, and half of the treatments with smaller $\min_i \hat{z}_{a,i}$ will be eliminated (Line 9). 

\textbf{Variance Equalization.} 
The principle of sampling rules in \texttt{BAI} is to equalize the variance of the estimate used in elimination. Other rules, such as Neyman allocation~\citep{neyman1992allocation}, have also been studied, and we discuss and compare them with our approach in Sec.~\ref{sec:appendix-neyman} of the appendix. If two eliminated arms have the same estimate, the probability that either is the best arm would be the same. In classic \texttt{MAB} with a single metric and homogeneous unit variance, the vanilla \texttt{SH} samples each active arm the same number of times $N_s$. Then, the empirical mean estimate $\hat{\mu}_s(a)$, which is used in elimination, will have the same variance for all arms, i.e., $\mathrm{Var}(\hat{\mu}_s(a)) = 1 / N_s$. If heterogeneous variances are considered~\citep{lalitha23variance}, uniform exploration no longer equalizes the variance of the empirical mean. Variance-based sampling is used to equalize the variance of the empirical mean estimator $\mathrm{Var}(\hat{\mu}_s(a))$, which samples each arm $N_s(a)$ times in proportion to their reward variance $\sigma_a^2$, i.e., $N_s(a) \varpropto \sigma_a^2$. Our problem is both multi-metric and uses the $z$ value instead of the reward means for elimination, so we follow the principle to design the relative-variance-based sampling to handle both challenges. 

\textbf{Sampling Rule in Single-Metric Model.} We first consider $M=1$ and omit the index $i$ for metrics. Since the estimate we use for treatment elimination is the $z$ values, we analyze the variance of $\hat{z}_s(a)$:
\begin{align*}
    \mathrm{Var}(\hat{z}_s(a)) = \frac{\rho_a^2}{N_s(a)} + \frac{\lambda_a^2}{N_s(0)} 
    \leq \frac{\rho_a^2}{N_s(a)} + \frac{\max_a\lambda_a^2}{N_s(0)}.
\end{align*}
Then, we perform a min-max optimization over the variance upper bound to equalize the variance as:
\begin{align*}
    \min_{\{N_s(a)\}} \max_{a\in\gA_s} \left(\frac{\rho_a^2}{N_s(a)} + \frac{\max_a\lambda_a^2}{N_s(0)}\right).
\end{align*}
Notice that only active treatments in $\gA_s$ will be pulled. Solving the problem results in an allocation as:
\begin{align*}
    N_s(0) \varpropto& ~  \max_{a \in \gA_s} \lambda_{a} \sqrt{\sum_{a\in \gA_s} \rho_{a}^2};\\
    N_s(a) \varpropto& ~ \rho_{a}^2, ~ \forall a\in \gA_s,
\end{align*}
which depends on relative variance and resembles the rule used in Lines 3-4 in algorithm~\ref{alg:mosh} if $M=1$.

\textbf{Sampling Rule in Multi-Metric Model.} To generalize the single-metric sampling rule to the multi-metric model, we first suppose the ``bottleneck'' metric $i^*_a$ is known a priori, and then we only need to consider the metric $i^*_a$ for each treatment $a$ and equalize its variance $\mathrm{Var}\left(\hat{z}_s(a, i^*_a)\right)$. This suffices to replace $\rho_a^2$ and $\lambda_a^2$ with $\rho_{a, i^*}^2$ and $\lambda_{a, i^*}^2$ in the sampling rule. 
However, $i^*_a$ is in reality not known before the sampling rule is designed, so we pursue a pessimistic surrogate to use the maximum relative variance $\max_{i} \rho_{a,i}$ and $\max_i \lambda_{a,i}$ to replace $\rho_{a, i^*}$ and $\lambda_{a, i^*}$, which results in our proposed sampling rule in algorithm~\ref{alg:mosh} as: 
\begin{align*}
    N_s(0) \varpropto& ~ \rho_{\gA_s, \Sigma} \cdot \lambda_{\gA_s, \Sigma};\\
    N_s(a) \varpropto& ~ \max_i\rho_{a, i}^2, ~ \forall a\in \gA_s.
\end{align*}
This rule equalizes $\max_i \mathrm{Var}\left(\hat{z}_s(a,i)\right)$ for treatments to safeguard against the worst case, where $i^*_a$ is the metric with the largest relative variance. 

\textbf{Discussion on Relative Variance.} The relative variance $\rho_{a,i}$ and $\lambda_{a,i}$ point out whether the difficulty of passing the validation is from the variance of the treatment or the variance of the control. If $\rho_{a,i}^2$ is large, the control will look like a ``deterministic" arm since its variance is much smaller. Moreover, for any treatment $a$, if there exists a metric $i$ with a large $\rho_{a,i}$, it should be allocated more samples so that its estimate is as accurate as other treatments, since passing the validation of metric $i$ requires estimating its mean reward more accurately. To understand terms $\rho_{\gA_s, \Sigma}$ and $\lambda_{\gA_s, \Sigma}$, we first consider the single metric model. Then $\rho^2_{\gA_s, \Sigma}$ becomes the sum of relative variances $\rho^2_a$ over all active treatments in $\gA_s$, representing the total randomness in $z$ value estimates from treatments, and $\lambda^2_{\gA_s, \Sigma}$ will be the maximum of $\lambda^2_{a}$, indicating the variance of $z$ value estimates from the randomness of control. Multi-metric models simply replace the relative variances with the maximum over metrics.

\section{Theoretical Results}\label{sec:theory}

In this section, we present the probability of error guarantee for \texttt{SHRVar}. For simplicity, we assume $A$ is a power of $2$. We first introduce some useful notations. For an arbitrary set $\gS$ of treatments, we define the heterogeneity of relative variance $\kappa_{\gS, a, i}$ for treatment $a$ and metric $i$ as:
\begin{align*}
    \kappa_{\gS, a, i} = \frac{\rho_{a,i}^2}{\max_i\rho_{a,i}^2}\frac{ \rho_{\gS, \Sigma}}{\rho_{\gS, \Sigma} + \lambda_{\gS,\Sigma}} + \frac{\lambda_{a,i}^2}{\max_{a \in \gS, i} \lambda^2_{a,i} }\frac{\lambda_{\gS,\Sigma}}{\rho_{\gS, \Sigma} + \lambda_{\gS,\Sigma}}.
\end{align*}
We define the effective gap $D_{\gS, a}>0$ compared to the best treatment in terms of its $z$ values as:
\begin{align*}
    D_{\gS, a}^2 = \min_{i\in [M]} \max_{j\in [M]} \frac{[z_{a^*,i} - z_{a,j}]_+^2}{(\kappa_{\gS, a,j} + \kappa_{\gS, a^*,i})^2}.
\end{align*}
Finally, for any set of treatments $\gS$, let $\gS'_{\ervc}$ be its subset excluding $1/4$ of the treatments with smallest $D_{\gS, a}$. The $H_3$ complexity is defined as follows:
$$
    H_3 = \left(\min_{\gS: a^*\in \gS} \frac{\min_{a\in\gS'_{\ervc}}D^2_{\gS, a}}{ \left( \rho_{\gS, \Sigma} + \lambda_{\gS,\Sigma} \right)^2}\right)^{-1}.
$$

\begin{theorem}\label{thm:estimate}
    If we use \texttt{SHRVar} in algorithm~\ref{alg:mosh} to identify the best treatment $a^*$ in our exploration phase, the probability of mistake can be upper bounded as:
    \begin{align*}
        \prob\left( \hat{a}\neq a^* \right) \leq 6 M \log_2 A \cdot \exp\left( -\frac{T}{2 H_3 \log_2 A} \right).
    \end{align*}
\end{theorem}

\textbf{Proof Roadmap.} The proof of the theorem is in the appendix Sec.~\ref{sec:proof-estimate} and builds on the sequential halving framework where we first fix a stage $s$ and the active treatment set $\gA_s$ which has not been eliminated, and then analyze the probability that at least half of the active treatments in $\gA_s$ will ``outperform'' $a^*$ in the elimination metric $\min_i \hat{z}_s(a,i)$. This requires analyzing the probability that each sub-optimal treatment $a$ outperforms $a^*$. For each treatment $a$, we first characterize the probability that the minimum $z$ value estimate is larger than $a^*$, i.e., $\min_i \hat{z}_s(a,i) > \min_i \hat{z}_s(a^*,i)$, which requires decoupling the dependence of different treatments when estimating the $z$ values since they all involve the empirical reward of control. This argument leads to the instance-dependent exponent $D_{\gA_s, a}$ analogous to the reward gap in classic bandit models. Then, we use Markov's inequality to characterize the probability of eliminating $a^*$ at each stage and use a union bound over the stages to reach the final result. Next, we discuss the instance-dependent parameters.

\textbf{Effective Gap.} 
For any set $\gS$, the effective gap $D_{\gS,a}$ measures the difficulty of eliminating a sub-optimal arm $a$ from it, which is analogous to the expected reward gap in \texttt{MAB}. Since $\kappa_{\gS, a,i} \leq 1$ by definition, we have:
\begin{align*}
    D_{\gS, a} =&\min_{i\in [M]} \max_{j\in [M]} \frac{[z_{a^*,i} - z_{a,j}]_+}{(\kappa_{\gS, a,j} + \kappa_{\gS, a^*,i})}
    \geq   \min_{i\in [M]} \max_{j\in [M]}[z_{a^*,i} - z_{a,j}]_+ \geq z_{a^*,i^*} - z_{a,i^*}.
\end{align*}
Thus, $D_{\gS, a}$ can be viewed as the gap between the $z$ value of the bottleneck metric $i^*_a$ of two treatments. If the $z$ values have a larger gap between $a^*$ and $a$, $D_{\gS, a}$ will be larger, and the quality of the treatment will be easier to distinguish. 

\textbf{Relative Variance Heterogeneity.} The heterogeneity $\kappa_{\gS, a, i}$ measures the error if we use the maximum relative variance $\max_i\rho_{a,i}$ to approximate each $\rho_{a,i}$ as in our sampling rule. If $\kappa_{\gS, a, i}$ is large and close to $1$, the heterogeneity of $\rho_{a,i}$ is small, and approximating each with $\max_i\rho_{a,i}$ would be accurate. 
If $\kappa_{\gS, a, i}$ becomes smaller, the relative variance $\rho_{a,i}$ of this metric would be much smaller than $\max_i\rho_{a,i}$, making this metric easy to estimate. Potentially, the hardness of identification would become smaller, resulting in a larger effective gap $D_{\gS, a}$.

\textbf{Complexity Measure $H_3$.}
The role of the two minimum arguments in the exponent $H_3^{-1}$ is to trade off the difficulty of distinguishing between more arms with a larger gap $D_{\gS, a}$ and fewer arms with a smaller gap $D_{\gS, a}$. If we choose a large treatment set $\gS$, more arms with a smaller gap will be excluded from $\gS_{\ervc}'$, which results in a larger numerator in the definition of $H_3^{-1}$. At the same time, more treatments are included in $\gS$, so the denominator will also be larger since the total relative variance $\rho_{\gS, \Sigma}$ and $\lambda_{\gS, \Sigma}$ increase. Therefore, the exponent is determined by the worst set $\gS$ with a moderate size, small gaps, and large relative variances. It is also possible that with a less refined analysis, we can obtain a looser but more explainable complexity measure $H_3'$ as follows:
\begin{align*}
    H_3 \leq& \left(\min_{\gS: a^*\in \gS} \frac{\min_{a \neq a^*}(z_{a^*, i^*} - z_{a,i^*})^2}{ \left( \rho_{\gS, \Sigma} + \lambda_{\gS,\Sigma} \right)^2}\right)^{-1} 
    \lesssim  \left(\frac{\min_{a}(z_{a^*,i^*} - z_{a,i^*})^2}{ \sum_{a\in\gA} \max_i \rho_{a,i}^2 + \max_{a\in \gA ,i} \lambda^2_{a,i} } \right)^{-1}
    \equiv H_3',
\end{align*}
where $\lesssim$ hides absolute constants. The form of $H_3'$ resembles the complexity measure used in \cite[Theorem 3]{lalitha23variance} for the theoretical analysis of \texttt{SHVar} where the minimum reward gap is replaced by the minimum $z$ value gap of the bottleneck metrics between the best treatment and sub-optimal treatments, and the sum of variances are replaced by the sum of maximum relative variances in the denominator. With our proof methodology, we can obtain a tighter error probability bound for \texttt{SHVar} with a complexity measure similar to $H_3$. We further discuss this complexity measure and achievable lower bounds in the appendix Sec~\ref{sec:appendix:discuss}.

\textbf{Comparation to \texttt{SH} and \texttt{SHVar}.} We simplify our model to a single metric with homogeneous variance among treatments. Then, the sum of relative variances in the denominator will be approximately $|\gS|$, and the effective gap will be the difference of $z$ values, i.e., $D_{\gS,a} = z_{a^*} - z_a$. If we change the index of treatments to order them from large to small $z$-values, the complexity measure $H_3$ becomes:
\begin{align*}
    H_3 = &\left(\min_{\gS: a^*\in \gS} \frac{\min_{a\in\gS'_{\ervc}}(z_{a^*} - z_a)^2}{ |\gS|}\right)^{-1}
    \lesssim  \max_{a \in \gA} \frac{a}{(z_{a^*} - z_a)^2},
\end{align*}
where $\lesssim$ hides absolute constants and it almost recovers $H_2$ in~\cite[Theorem 4.1]{karnin13bai} for \texttt{SH}, except for replacing the reward gap with the $z$ value gap. We can also simplify our model to a single metric with heterogeneous variance and compare with~\citep{lalitha23variance}. The looser exponent $H_3'$ is simplified as:
\begin{align*}
    H_3\lesssim H_3' =  \left(\frac{\min_{a}(z_{a^*} - z_a)^2}{ \sum_{a\in\gA} \rho_a^2 + \max_a \lambda_a^2 }\right)^{-1},
\end{align*}
which almost recovers~\cite[Theorem 3]{lalitha23variance} for \texttt{SHVar} except for replacing the reward gap with the $z$-value gap and variance with relative variance. The two reductions show that \texttt{SHRVar} is a generalization of both \texttt{SH} and \texttt{SHVar} to the multi-metric with validation, and does not suffer performance loss in simpler models, which also shows $H_3$ provides a tighter characterization.

\subsection{Discussions on the Elimination Rule}
\begin{algorithm}[t]
\caption{Sequential Halving with Relative Variance and Confidence-Based Elimination (\texttt{SHRVar-c})}\label{alg:shrvarc}
\begin{algorithmic}[1]
\STATE initialize $\gA_1 \leftarrow \mathcal{A}/\{0\}$;
\FOR{$s=1,2,\cdots, \lceil \log_2 A\rceil$}
    \STATE for control, $N_s(0) =\left \lfloor \frac{ \lambda_{\gA_s, \Sigma} }{ \rho_{\gA_s, \Sigma} + \lambda_{\gA_s, \Sigma}} \frac{T}{\lceil \log_2 A\rceil} \right\rfloor$.
    \STATE for active treatments $a\in \gA_s$, 
    $$N_s(a) = \left \lfloor \frac{\max_i\rho_{a,i}^2}{\rho_{\gA_s, \Sigma} \left( \rho_{\gA_s, \Sigma} + \lambda_{\gA_s, \Sigma} \right) } \frac{T}{\lceil \log_2 A\rceil} \right\rfloor;$$
    \FOR{$a\in \gA_s \cup \{0\}$}
        \STATE sample $a$ for $N_s(a)$ times and obtain reward vectors $\rvx_{a, 1}, \rvx_{a,2}, \cdots, \rvx_{a, N_s(a)}$;
    \ENDFOR
    \STATE compute $\hat{\mu}_s(a, i)$ and $\hat{z}_{s}(a, i)$ as equation~\ref{eq:empirical-mean}.
    \FOR{$a\in \gA_s$}
            \STATE compute $\delta_s(a)$ as follows:
            \begin{align*}
                \delta_s(a) = \inf\left\{\delta\left| \mathrm{UCB}_s^\delta(a) \leq \max_{a'\in \gA_s} \mathrm{LCB}_s^\delta(a') \right.\right\};
            \end{align*}
        \ENDFOR
        \STATE Let $\gA_{s+1}$ be the set of $\lceil|\gA_s|/2\rceil$ arms with smaller $\delta_s(a)$;
\ENDFOR
\STATE recommend the arm in $\mathcal{A}_{\lceil \log_2 A \rceil+1}$.
\end{algorithmic}
\end{algorithm}
As for the elimination rule, we note that the minimum $z$ value estimator over metrics $\min_i \hat{z}_s(a,i)$ used in \texttt{SHRVar} may be subject to underestimation since $\E[\min_i \hat{z}_s(a,i)] \leq \min_i z_{a,i}$, and create bias in treatment elimination. In our design of \texttt{SHRVar}, the variances of different treatments and metrics are approximately equalized through our sampling rule, which mitigates the underestimation effect. However, to study this effect, especially under the circumstances where the sampling strategy is not controlled, we consider a confidence-based elimination rule as follows. Let $b_s^\delta(a, i)$ be the confidence bonus when the out-of-concentration probability is $\delta$ as follows:
\begin{align*}
    b_s^\delta(a, i) = 2\sqrt{\left(\frac{\rho_{a,i}^2}{N_s(a)} + \frac{\lambda_{a,i}^2}{N_s(0)} \right) \log \left( \frac{|\gA_s| M}{\delta} \right)}.
\end{align*}
Let the following be the upper and lower confidence bounds where $\min_i z_{a,i}$ falls into with probability at least $1-\delta$:
\begin{align*}
    \mathrm{LCB}_s^\delta(a) =& \min_i \left\{ \hat{z}_s(a,i) - b_s^\delta(a,i) \right\},\\
    \mathrm{UCB}_s^\delta(a) =& \min_i \left\{ \hat{z}_s(a,i) + b_s^\delta(a,i) \right\}.
\end{align*}
For each treatment $a$, we calculate the confidence level $\delta_s(a) = \delta$ such that $\mathrm{UCB}_s^\delta(a)$ equals the largest lower confidence bound over treatments, i.e., $\max_a \mathrm{LCB}_s^\delta(a)$. Then, the confidence-based elimination rule eliminates half of the arms with smaller $\delta_s(a)$ at each stage. The algorithm, denoted as \texttt{SHRVar-c}, is in algorithm~\ref{alg:shrvarc} and more detailed discussions of confidence-based elimination are discussed in Sec.~\ref{sec:conf} of the appendix. We also provide a theoretical guarantee similar to Theorem~\ref{thm:estimate}:
\begin{theorem}\label{thm:conf}
    If we use \texttt{SHRVar-c} in algorithm~\ref{alg:shrvarc} to identify the best treatment $a^*$ in our exploration phase, the probability of mistake can be upper bounded as:
    \begin{align*}
        \prob\left( \hat{a}\neq a^* \right) \leq 6 M \log_2 A \cdot \exp\left( -\frac{T}{2 \tilde{H}_3 \log_2 A} \right).
    \end{align*}
\end{theorem}
Here, $\tilde{H}_3$ is a complexity notion similar to $H_3$, and its explicit form and discussions are in Appendix~\ref{sec:conf}.

\section{Numerical Results}\label{sec:experiment}

\begin{figure}[t]
    \centering
    \subfigure[Exploration Accuracy]{
        \centering
        \includegraphics[width=0.48\linewidth]{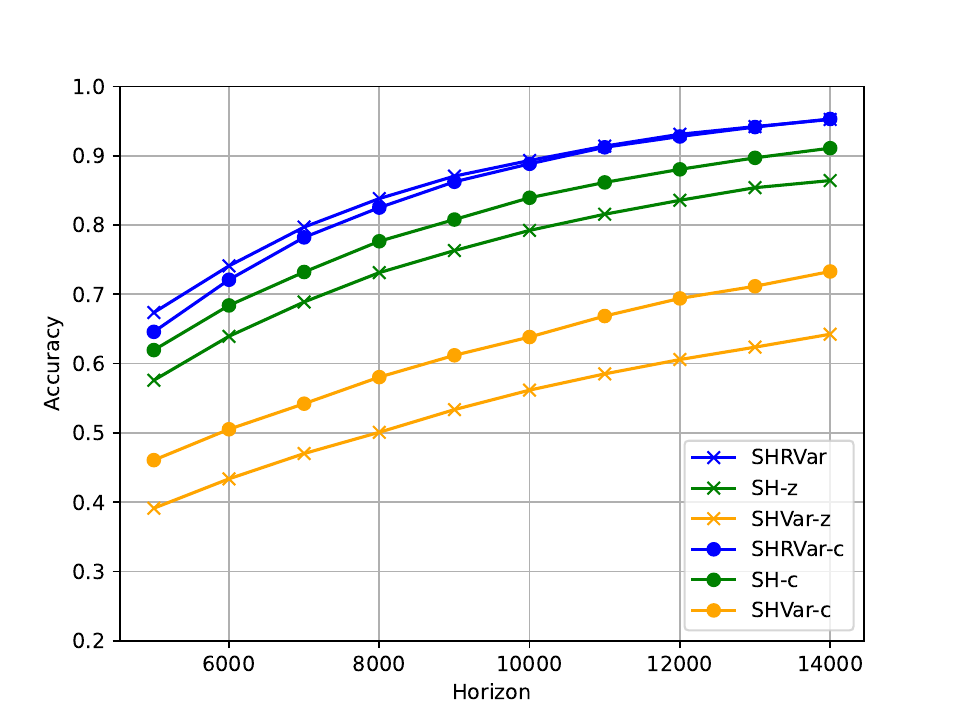}
        \label{fig:err}
    }
     \subfigure[Validation Success Probability]{
        \centering
        \includegraphics[width=0.48\linewidth]{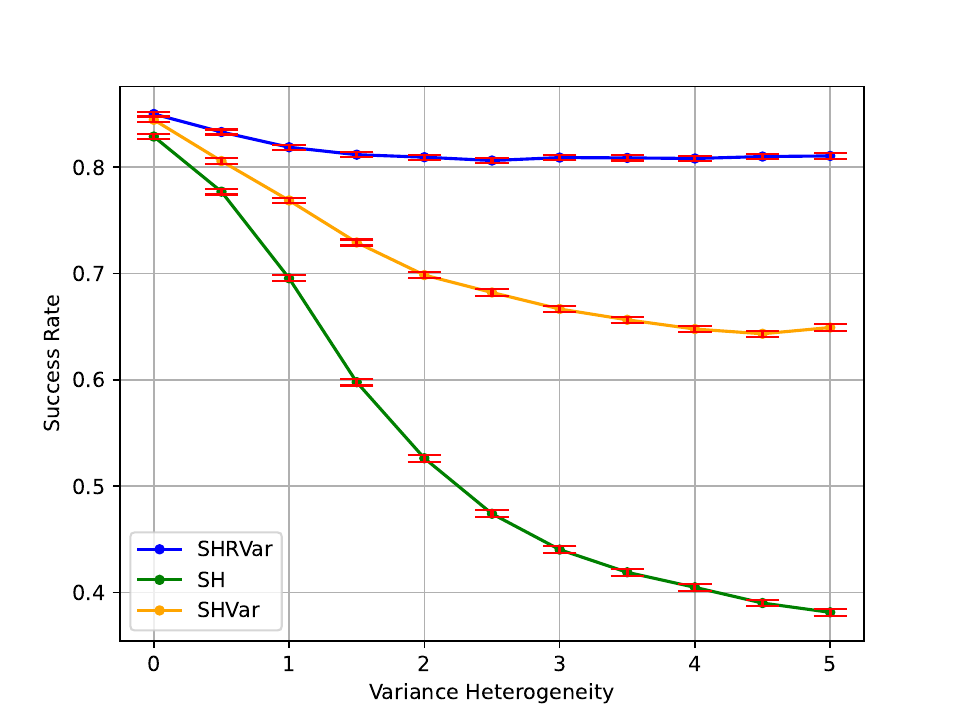}
        \label{fig:val}
    }
    \caption{(a) Exploration Accuracy: rate of identifying the best treatment in exploration as $T$ increases. (b) Validation Success Probability: rate of passing validation, i.e., successfully validating a better-than-control treatment. Error bars represent confidence intervals.}
\end{figure}

We demonstrate the empirical performance of \texttt{SHRVar}. We first study the exploration phase and compare the accuracy of identifying $a^*$. Then, we experiment with both the exploration and validation phases to show that \texttt{SHRVar} significantly increases the probability of validating a better treatment, especially in heterogeneous variance instances. This demonstrates the superiority of $z$ value estimates over expected reward in \texttt{AED} with validation. The results are averaged over $10^5$ repetitions. We report the error probability $\prob(\hat{a} = a^*)$ that the algorithm successfully recommends $a^*$ in Fig.~\ref{fig:err}, and the probability $\prob(\gE_{\rv})$ that the recommended treatment $\hat{a}$ passes validation in Fig.~\ref{fig:val}. A larger-scale experiment is conducted in Sec.~\ref{sec:appendix-experiment} of the appendix, where the findings are exactly the same.

\subsection{Error Probability in Exploration}
We consider a bandit instance with $A = 16$ treatments, and $M=3$ metrics. Suppose we use non-Bayesian validation, and the confidence level $\delta_i$ is the same for each metric. So the validation constant $\xi_{a,i}$ would be the same for all $(a,i)$. The expected reward and variance of control are set to be $[0,0,0]$ and $[1,1,1]$, respectively, for all metrics. For the treatments, the expected reward of the best treatment $a^*$ is randomly chosen as $[2.4697, 1.5556, 1.1180]$, and all sub-optimal treatments have the same expected reward $[2.0125, 1.6971, 1.3416]$. All treatments have the same variances $[4, 1, 0.25]$.

\textbf{Baselines.} Since fixed-budget BAI algorithms in the literature, such as \texttt{SH} and \texttt{SHVar}, do not deal with multi-metric problems, we adapt them to \texttt{M3AB} by revising their elimination strategies to either the $z$-value-based or confidence-based eliminations proposed in algorithm~\ref{alg:mosh} and algorithm~\ref{alg:shrvarc}. Together with \texttt{SHRVar-c}, this constructs $5$ baselines where the rest are called \texttt{SH-z}, \texttt{SH-c}, \texttt{SHVar-z}, and \texttt{SHVar-c}. Specifically, to study the influence of relative-variance-based sampling, \texttt{SH-z} replaces Lines 3-4 of algorithm \ref{alg:mosh} with uniform sampling and \texttt{SHVar-z} replaces it with a variance-based sampling rule similar to \texttt{SHVar}, which samples each arm proportional to $\max_i \sigma_{a,i}^2$. The comparison between the \texttt{SH}, \texttt{SHVar}, and \texttt{SHRVar} under the same $z$-value-based elimination rule reflects the efficiency of the relative-variance-based sampling rule. Additionally, \texttt{SH-c} and \texttt{SHVar-c} are created by replacing the elimination rule in \texttt{SH-z} and \texttt{SHVar-z} with the confidence-based elimination. Under each sampling rule, i.e., uniform sampling in \texttt{SH}, variance-based sampling in \texttt{SHVar}, and relative-variance-based sampling in \texttt{SHRVar}, we could compare the performance of the elimination rules and measure how the potential under-estimation effect of the $\min_i \hat{z}_{a,i}$ estimator would take place in different levels of variance equalization.

\textbf{Performance.} For all time horizons $T$, the relative-variance-based sampling strategy significantly improves over uniform sampling and variance-based sampling used in \texttt{SHVar}, with \texttt{SHRVar} and \texttt{SHRVar-c} both having a better error probability than other baselines. This is because the relative-variance-based sampling rule balances the estimated variance across treatments and avoids the error from a single uncertain estimate. Therefore, relative-variance-based sampling should be used to achieve the best efficiency. We also observe that with a small horizon, \texttt{SHVar} slightly performs better than \texttt{SHVar-c}. The difference disappears as the total horizon increases, which is consistent with our theory as $H_3$ and $\tilde{H}_3$ become the same in large $T$.

\textbf{Robustness.} On the other hand, if we fix the elimination rule and evaluate the performance drop from \texttt{SHRVar} to \texttt{SHVar} and \texttt{SH}, we can observe that under the confidence-based elimination strategy, the performance degradation is much milder, i.e., the performance difference between \texttt{SHRVar-c} and \texttt{SH-c} is much smaller than the difference between \texttt{SHRVar} and \texttt{SH}. This shows that the vanilla $z$-value estimate indeed suffers from the under-estimation effect. Even though it is not an issue when the variances of treatments are approximately equalized with the relative-variance-based sampling, it incurs a non-negligible performance loss when the sampling rule deviates from optimal. 
This demonstrates the robustness of the confidence-based elimination rule, which is useful in practice. For example, some system constraints may prevent the use of unbalanced sampling fractions for each active treatment, i.e., uniform sampling is required due to fairness. In other cases, the variances $\sigma_{a,i}^2$ also come from estimates, which may be inaccurate or subject to non-stationarity. So, a perfect $z$-estimate variance equalization is very difficult. More discussions are in Appendix~\ref{sec:conf}.

\subsection{Validation Success Probability}
We simplify our model to $M=1$ and experiment with both exploration and validation phases to show that \texttt{SHRVar} significantly improves the chance of passing validation compared to the vanilla \texttt{SH} proposed in \citep{karnin13bai} and \texttt{SHVar} proposed in \citep{lalitha23variance}, especially under heterogeneous variance instances. Both \texttt{SH} and \texttt{SHVar} use the estimate of expected reward for treatment elimination instead of the $z$ value used in \texttt{SHRVar}.
We consider $A=27$ treatments where the control has an expected reward $\mu_0 = 0$ and variance $\sigma_0^2 = 1$. To make the results more comparable, we fix the $z$ value of each treatment with $z_a = 0.3 - 0.1 \sqrt{a}$ so the first treatment is the best. Then, we gradually increase the level of variance heterogeneity $l$, which is a constant in $[0,5]$. The standard deviation $\sigma_a$ of each treatment $a$ is set as $1 + a^{l}$, so more sub-optimal arms have larger variance. We set the horizons for the exploration and validation phases to be $500$ and $100,$ respectively. It can be observed that when the variance heterogeneity is small, e.g., $l\leq 0.5$, all algorithms have relatively good performance. However, as the variance heterogeneity $l$ over the treatments increases, our proposed algorithm \texttt{SHRVar} maintains a good validation success rate over $80\%$, but the performances of both other baselines quickly decrease. This is because treatments with a larger variance will be hard to validate, which is not taken into account by the reward estimate used in \texttt{SH} and \texttt{SHVar}. The results demonstrate the superiority of $z$ value estimation in \texttt{AED} with validation. 
Notice that \texttt{SH} is also worse than \texttt{SHVar} since it is completely variance-agnostic and thus not robust to variance heterogeneity.

\section{Conclusion}
In this paper, we consider \texttt{M3AB} with both exploration and validation to find the treatment with the best chance to pass an A/B test against control.
We characterize the best treatment through $z$ values and propose \texttt{SHRVar} with a novel relative-variance-based sampling rule. Theoretical analysis shows that \texttt{SHRVar} has an exponentially decreasing error probability, generalizing \texttt{SH} or \texttt{SHVar}, which also works for bounded or sub-Gaussian rewards. Numerical results also demonstrate its superiority. 
Our future directions include incorporating unknown variance estimations and generalizing our algorithms and results for linear bandits.

\section*{Acknowledgment}
We thank Cathy Jiao for her support of this project.

\bibliography{ref}
\bibliographystyle{unsrtnat}

\newpage
\appendix

\section{Additional Experiments}\label{sec:appendix-experiment}
In this section, we provide results from additional numerical experiments that demonstrate our proposed approach's generality against multiple factors such as a large number of treatments, heterogeneous variance, and unknown variances. All results are evaluated over $10^5$ repetitions. All experiments in this paper are conducted on a single CPU with less than an hour of runtime.

\subsection{Generality to Heterogeneous Variance and Number of Treatments}
\begin{figure}[t]
    \centering
    \subfigure[Exploration Accuracy]{
        \centering
        \includegraphics[width=0.48\linewidth]{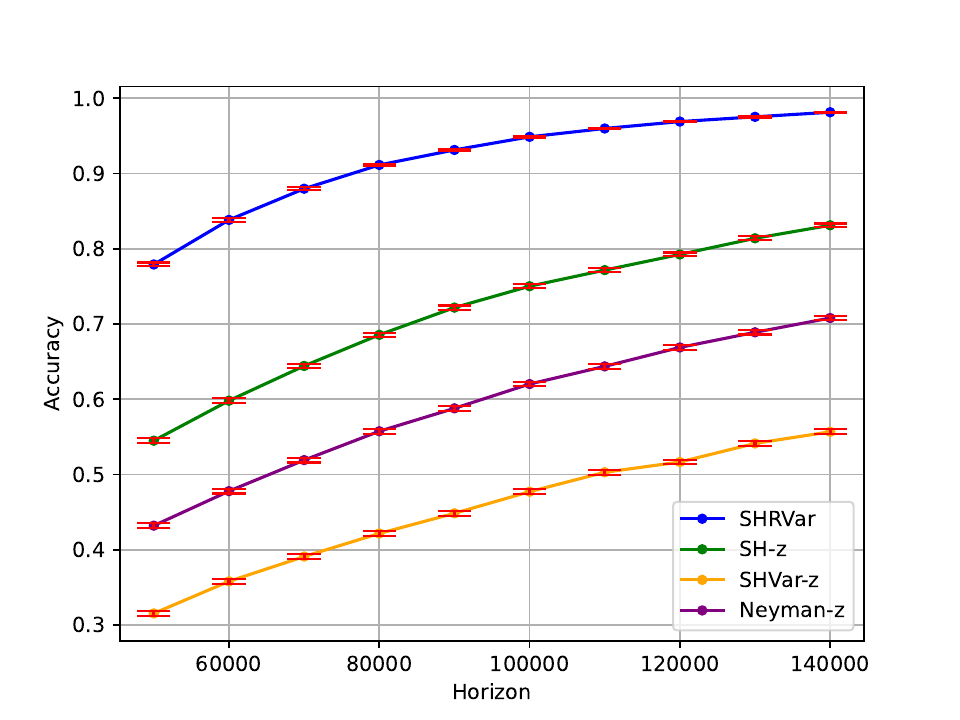}
    }
    \subfigure[Validation Success Probability]{
        \centering
        \includegraphics[width=0.48\linewidth]{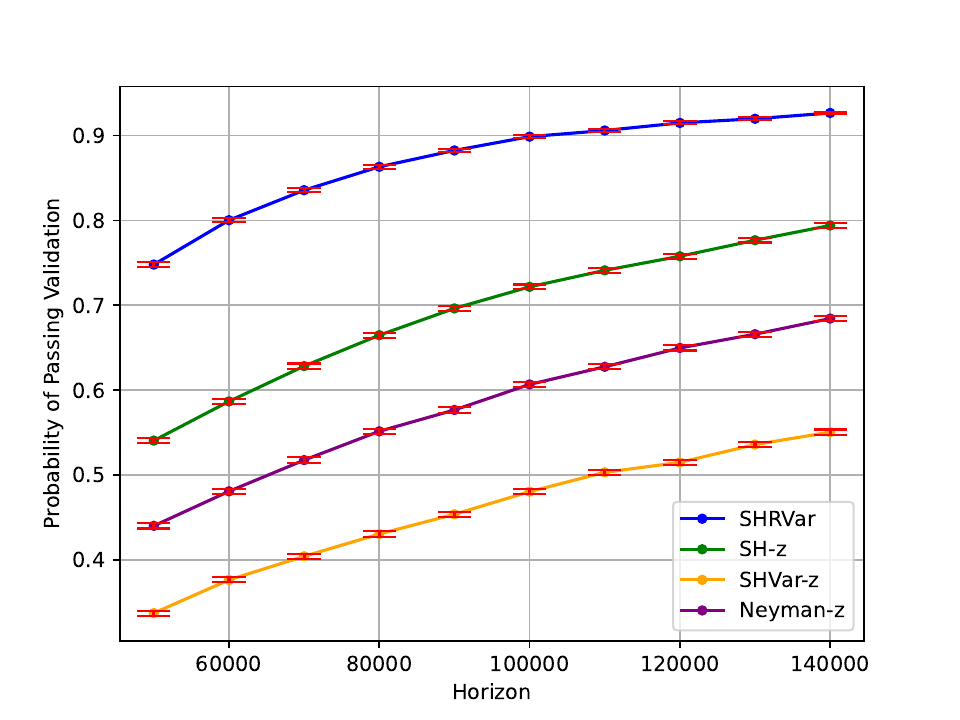}
    }
    \caption{(a) Exploration Accuracy: rate of identifying the best treatment in exploration as time horizon $T$ increases. (b) Validation Success Probability: rate of passing validation, i.e., successfully validating a better-than-control treatment after both phases. Error bars represent confidence intervals.}
    \label{fig:128-exist}
\end{figure}

We consider an \texttt{M3AB} problem with $A=128$ treatments and $M=3$ metrics, where each treatment-metric pair has a different variance. The reward of control is set to $[0,0,0]$ for all metrics, and the variance of control is $[1,1,1]$. For treatments, we set the $z$-value for the best treatment to be $[0.15,0.15,0.05]$ for the three metrics, and set the $z$-value for other treatments to be $[-0.05, 0.25, 0.25]$. Therefore, the only treatment that is better than control for all three metrics is the optimal treatment. The relative variances of treatments for the three metrics are set to $[0.8, 0.5, 0.2]$ plus a random noise uniformly distributed in $[-0.1,0.1]$. We use the same validation constant $\xi_{a,i}$ for all pairs of treatment and metric. Then, we calculate the standard deviation and the reward mean of each arm inversely from $z$-values and relative variances, and the variances will be heterogeneous. The validation time horizon $T_{\rv} = 2000$. We compare our proposed algorithm \texttt{SHRVar} to three baselines, i.e., \texttt{SH-z} and \texttt{SHVar-z} introduced in Sec.~\ref{sec:experiment}, and the baseline \texttt{Neyman-z} which replaces the sampling strategy of \texttt{SHVar-z} with Neyman allocation~\citep{neyman1992allocation}. The detailed discussion of Neyman allocation and the variance equalization allocation that we use in our paper will be provided in the appendix Sec.~\ref{sec:appendix-neyman}. The results are shown in Fig.~\ref{fig:128-exist}. The result shows that for both the probability of identifying the best treatment $a^*$ and the probability of passing validation, our proposed \texttt{SHRVar} has the best performance compared to all baselines, including Neyman allocation, which is asymptotically optimal in two-armed settings. The trend coincides with the results shown in Sec.~\ref{sec:experiment}, which is for a smaller number of arms and variances only heterogeneous among metrics. Therefore, it shows the generality of \texttt{SHRVar} and the advantage that \texttt{SHRVar} is robust to the number of arms and the level of heterogeneity between metrics and arms. 

\subsection{Robustness to Unknown Variance}

\begin{figure}[t]
    \centering
    \subfigure[Exploration Accuracy]{
        \centering
        \includegraphics[width=0.48\linewidth]{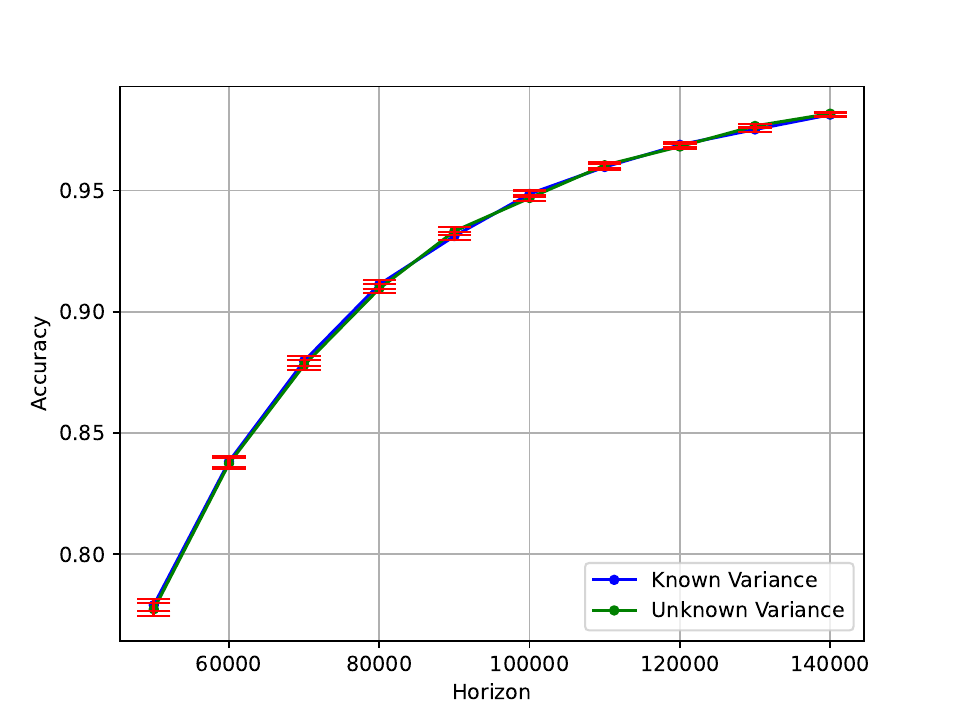}
    }
    \subfigure[Validation Success Probability]{
        \centering
        \includegraphics[width=0.48\linewidth]{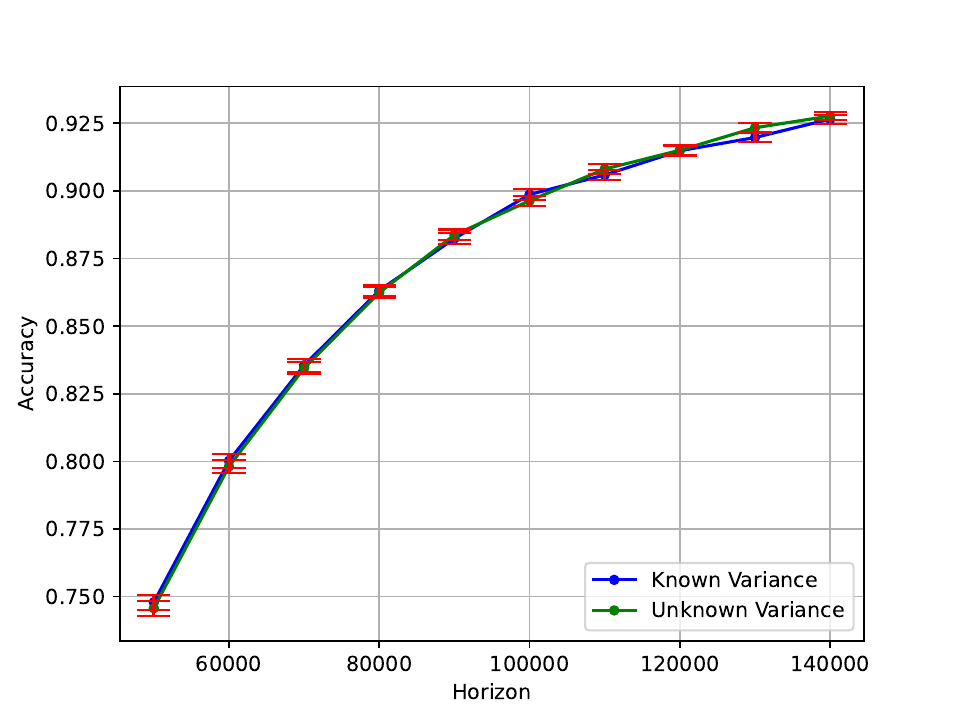}
    }
    \caption{(a) Exploration Accuracy: rate of identifying the best treatment in exploration as time horizon $T$ increases. (b) Validation Success Probability: rate of passing validation, i.e., successfully validating a better-than-control treatment after both phases. Error bars represent confidence intervals.}
    \label{fig:unknown}
\end{figure}

In this section, we incorporate \texttt{SHRVar} into the unknown variance setting and compare the performance to show the robustness of the proposed approach. We call the variant \texttt{SHRVar-Ada}. We assume the algorithm does not know the variance of each arm at the beginning, and therefore we first sample each arm uniformly and collect $T/(A+1)\log_2 A$ samples for each arm. Then, we estimate the variance of each arm using the sample variance, i.e.,
\begin{align*}
    \hat{\mu}_0(a,i) = \frac{1}{N_0(a)} \sum_{k=1}^{N_0(a)} \rvx_{a,k}, \quad \hat{\sigma}_{a,i}^2 = \frac{1}{N_0(a)-1} \sum_{k=1}^{N_0(a)}\left( \rvx_{a,k} - \hat{\mu}_0(a,i) \right),
\end{align*}
where $\rvx_{a,i, k}$ is the random reward of metric $i$ and treatment $a$ for the $k$-th pull, $N_0(a) = T/(A+1)\log_2 A$ for all arm, and $\hat{\mu}_0(a,i)$ is the empirical mean after the uniform pulling round. Then, we plug in the estimated $\hat{\sigma}_{a,i}$ to \texttt{SHRVar} with the remaining budget and proceed as if it is the true variance. The probability of finding the best treatment, and the probability of passing validation in the above \texttt{M3AB} setting are shown in Fig.~\ref{fig:unknown}. It shows that \texttt{SHRVar-Ada} in the unknown variance setting has almost the same performance as \texttt{SHRVar} of the known variance case, in both the probability of identifying the best treatment $a^*$ and the probability of passing validation. As the confidence intervals are overlapping with each other, given that the length of the interval is already small, we conclude that there is no significant difference between the unknown-variance variant and the original \texttt{SHRVar} algorithm. Therefore, it shows the robustness of \texttt{SHRVar} towards unknown variance.

\subsection{Type I Error}
\begin{figure}[t]
    \centering
    \subfigure[Standard]{
        \centering
        \includegraphics[width=0.48\linewidth]{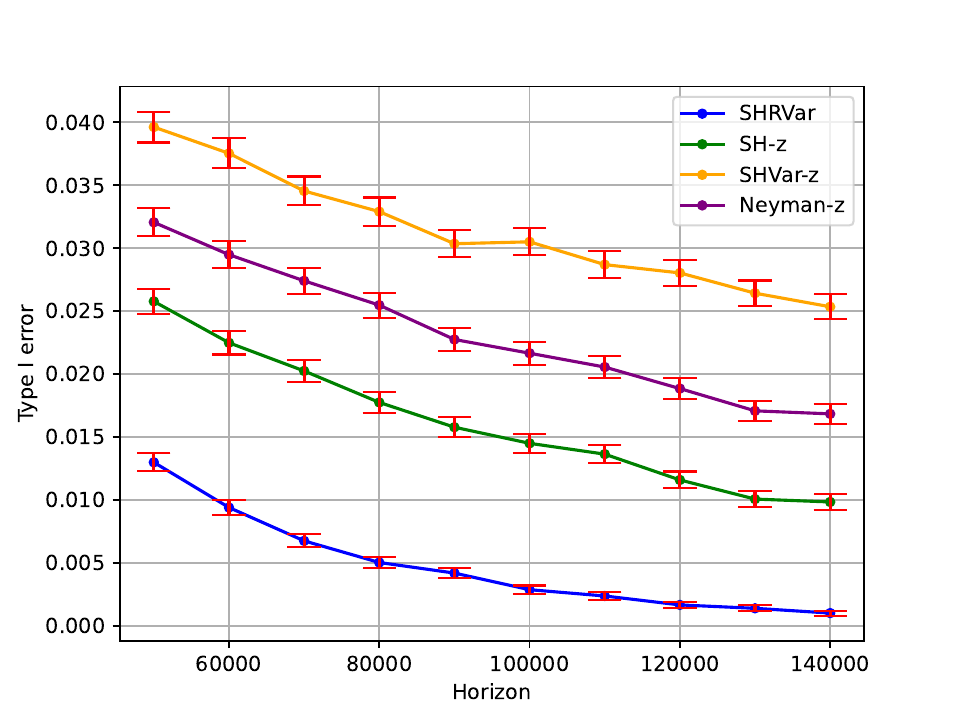}
    }
    \subfigure[Null]{
        \centering
        \includegraphics[width=0.48\linewidth]{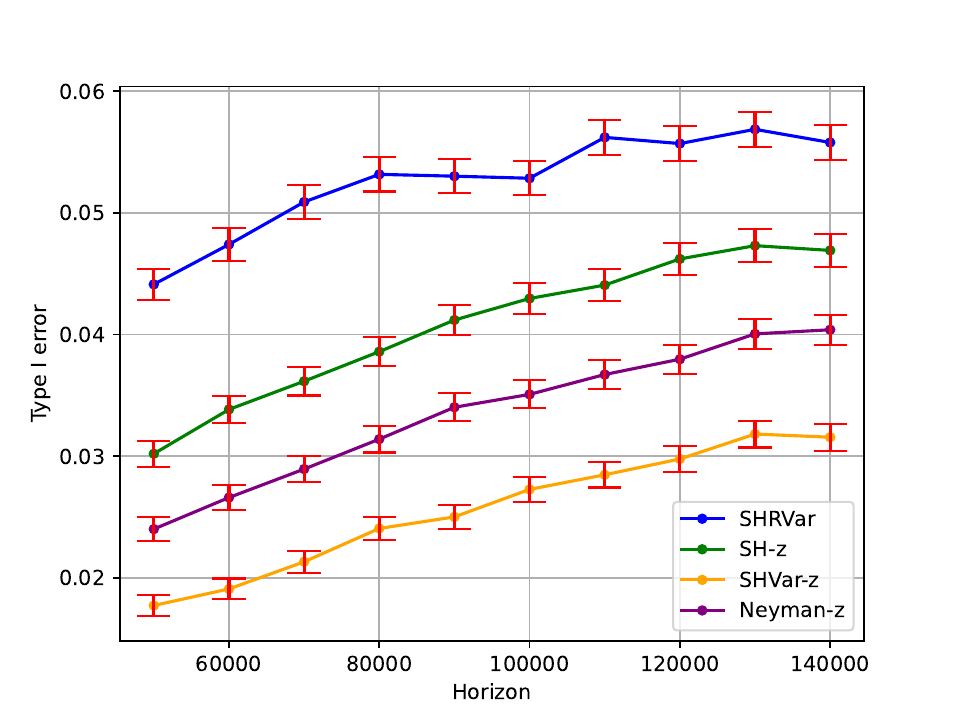}
    }
    \caption{(a) Standard instance: the best treatment is better than control in all metrics. (b) Null instance: no treatment is better than control uniformly.}
    \label{fig:128-typeI}
\end{figure}
It is also unsatisfactory if a worse-than-control arm is recommended and passes the validation against control, especially in settings where no treatment is better than control in all metrics. This is referred to as the Type I error. In this section, we study the type I error of the proposed algorithm compared to baselines. In addition to the settings mentioned above, where the best treatment is better than control for all metrics, we consider a null setting where all treatments are worse than control in at least one metric. This instance can be constructed from decreasing all the $z$-values of each treatment by $0.1$ in the above example, and the new $z$-value for the best treatment is $[-0.05,0.05,0.05]$, and for other treatments is $[-0.15,0.15,0.15]$. The type I error of both instances is shown in Fig.~\ref{fig:128-typeI}. It is shown that if there exists a treatment which is uniformly better than the control, our proposed algorithm \texttt{SHRVar} will have a type I error smaller than all baselines. This is because our algorithm will recommend the treatment with a higher probability of passing validation. If there exist treatments uniformly better than control, it will try to recommend the one in this set with the largest positive effect gap over the bottleneck metric. Therefore, it automatically decreases the type I error. If there is no treatment better than the control for all metrics, our algorithm will have a slightly larger type I error. This is because we try to recommend the treatment that has a better effect, i.e., a smaller effect gap to control, even though it is inferior. Naturally, other algorithms have slightly better type I error since they try to recommend treatments with a more negative effect more often, which is easy to distinguish for A/B tests. We emphasize that the type I error is mainly controlled by the A/B test in validation instead of the sampling and recommendation rules designed in \texttt{SHRVar}. For example, it is controlled by the required p-value to pass validation in non-Bayesian settings or the posterior probability threshold in Bayesian validation. Therefore, if worse algorithms return treatments with more negative effects, the type I error will be lower since the negative effect gap to control is larger. On the other hand, \texttt{SHRVar} returns arms with effects close to zero, so the type I error is close to the nominal rate defined by the p-value. Our experiments show that the type I error of \texttt{SHRVar} is not significantly worse than baselines in the case where no treatment is better than control, and we can control the type I error by setting the p-value or the posterior threshold properly.

\section{Variant of \texttt{SHRVar} with Confidence-Based Elimination}\label{sec:conf}
\begin{figure}[t]
    \centering
    \subfigure[overlap, $\delta = 0.01$]{
        \centering
         \includegraphics[width=0.3\linewidth]{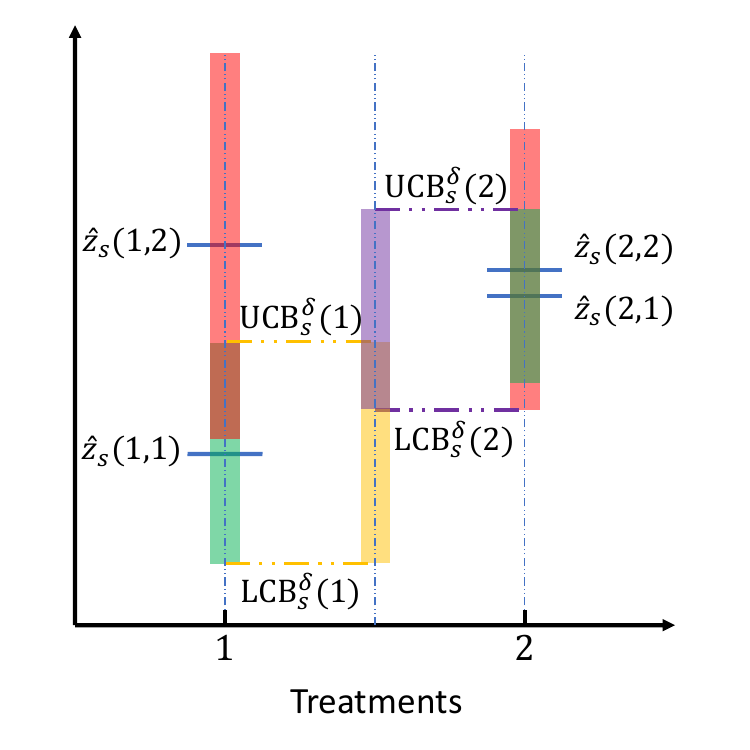}
         }
     \subfigure[almost overlap, $\delta = 0.05$]{
        \centering
         \includegraphics[width=0.3\linewidth]{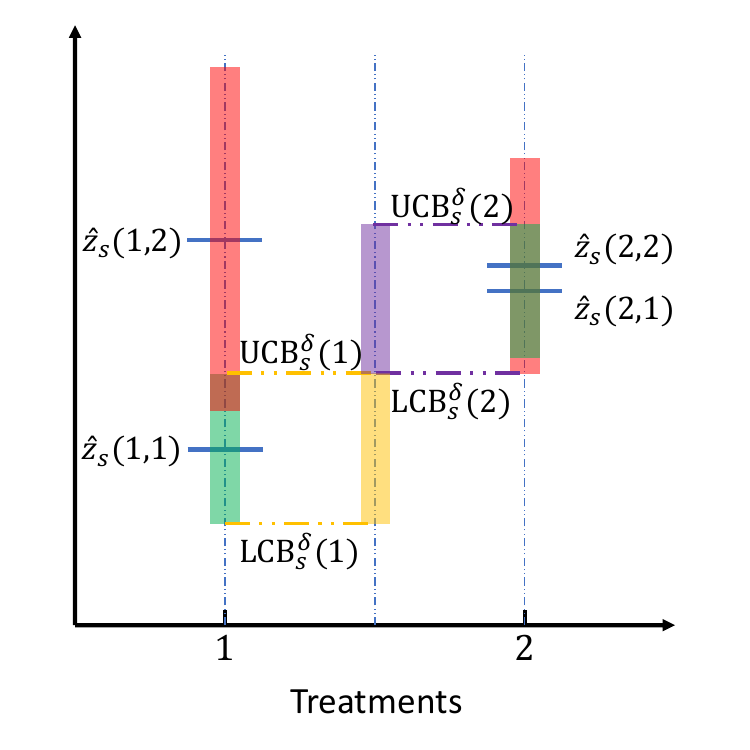}
         }
     \subfigure[separate, $\delta = 0.1$]{
    \centering
     \includegraphics[width=0.3\linewidth]{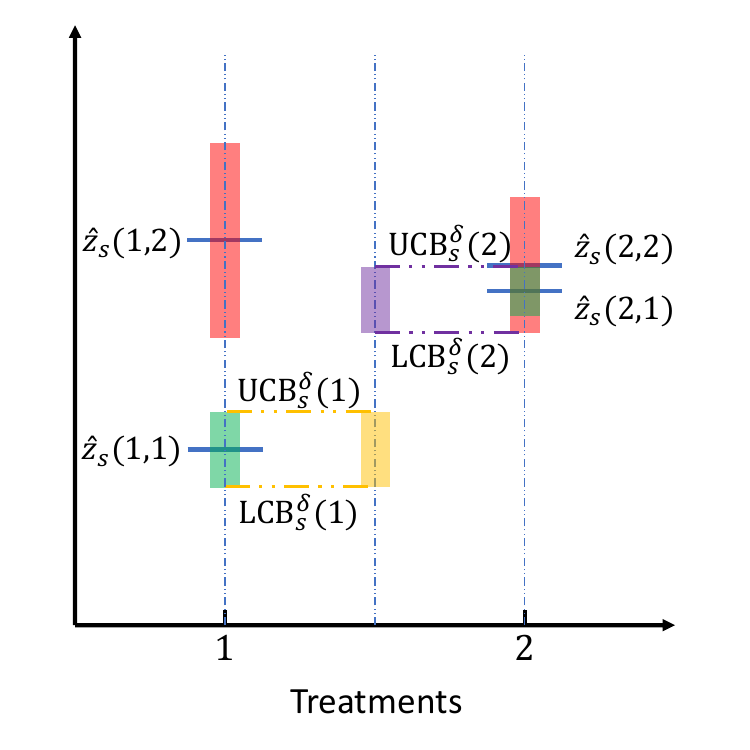}
     }
    \caption{A toy example to compute $\delta_s(1)$ where treatment $2$ has the largest LCB, i.e., $\argmax_{a} \mathrm{LCB}_s^\delta(a) = 2$. Empirical $z$ values are marked with solid lines, and confidence intervals are marked with color bars. In (a), $\delta$ is too small so that $\mathrm{UCB}_s^{\delta}(1)> \mathrm{LCB}_s^{\delta}(2)$. In (c), $\delta$ is too large so that $\mathrm{UCB}_s^{\delta}(1) < \mathrm{LCB}_s^{\delta}(2)$. In (b), we have $\mathrm{UCB}_s^{\delta}(1) = \mathrm{LCB}_s^{\delta}(2)$, meaning $\delta_s(1) = 0.05$.}
    \label{fig:confidence}
\end{figure}

In this section, we study and discuss the variant of the proposed \texttt{SHRVar} algorithm, the \texttt{SHRVar-c} algorithm with confidence-based elimination. The details of the algorithm are summarized in algorithm \ref{alg:shrvarc}. The only difference between \texttt{SHRVar-c} and \texttt{SHRVar} is the elimination strategy in Lines 9-12 that computes the confidence level $\delta_s(a)$ for each treatment and uses it to eliminate treatments. 

\subsection{Confidence-Based Elimination}
\textbf{Intuition.} 
We illustrate the intuition of this elimination rule through a simple example shown in Fig.~\ref{fig:confidence}. Recall that for each $z$ value estimate $\hat{z}_s(a,i)$, we can construct its confidence interval so that with probability $1-\delta$, we have:
\begin{align*}
    \underbrace{\min_i \left\{ \hat{z}_s(a,i) - b_s^\delta(a,i) \right\}}_{\mathrm{LCB}_s^\delta(a)} \leq \hat{z}_s(a,i) \leq \underbrace{\min_i \left\{ \hat{z}_s(a,i) + b_s^\delta(a,i) \right\}}_{\mathrm{UCB}_s^\delta(a)},
\end{align*}
where the confidence bonus $b_s^\delta(a,i)$ can be explicitly computed as:
\begin{align*}
    b_s^\delta(a, i) = 2\sqrt{\left(\frac{\rho_{a,i}^2}{N_s(a)} + \frac{\lambda_{a,i}^2}{N_s(0)} \right) \log \left( \frac{|\gA_s| M}{\delta} \right)}.
\end{align*}
Then, for any confidence level $\delta$ and any treatment $a$, if there exists another treatment $a'$ that has a lower confidence bound $\mathrm{LCB}_s^\delta(a')$ which is larger than the upper confidence bound $\mathrm{UCB}_s^\delta(a)$ of treatment $a$, it means the ground-truth $z$ value $z_{a,i}$ is smaller than $z_{a', i}$ with probability at least $1-\delta$. In other words, if $\mathrm{UCB}_s^\delta(a) < \max_{a'} \mathrm{LCB}_s^\delta(a')$, this means treatment $a$ is not the best treatment with high probability and eliminating it would incur an error probability at most $\delta$. This situation is illustrated in Fig.~\ref{fig:confidence}(c) where the yellow confidence interval of treatment $1$ does not overlap with the purple confidence interval of treatment $2$. However, if we decrease $\delta$ in this case while keeping the confidence intervals of both treatments non-overlapping, we can find a tighter probability of mistake upper bound. The optimal $\delta$ should be the case shown in Fig.~\ref{fig:confidence}(b), where the yellow and purple confidence intervals for both treatments almost overlap with each other, and we record this confidence level as $\delta_s(a)$:
\begin{align*}
    \delta_s(a) = \inf\left\{\delta \in \sR_+ \left| \mathrm{UCB}_s^\delta(a) \leq \max_{a'\in \gA_s} \mathrm{LCB}_s^\delta(a') \right.\right\}.
\end{align*}
Therefore, $\delta_s(a)$ represents the tightest error probability upper bound if we eliminate treatment $a$. Then, it becomes natural that we should eliminate the treatments with less probability of making mistakes since we are more confident that they are not the best treatment, which is the treatments with less $\delta_s(a)$.

\textbf{Underestimation Alleviation.} It is shown in our numerical experiment in Fig.~\ref{fig:err} that \texttt{SHRVar} may suffer performance loss due to underestimating the minimum $z$ value $\min_i z_{a,i}$ if the sampling strategy is not optimally designed, which deviates from equalization of the estimate variance. To be specific, since we use $\min_i \hat{z}_s(a,i)$ to estimate the ground-truth $\min_i z_{a,i}$, we have:
\begin{align*}
    \E[\min_i \hat{z}_s(a,i)] \leq \min_i \E[\hat{z}_s(a,i)] = \min_i z_{a,i}.
\end{align*}
If all the estimates $\hat{z}_s(a,i)$ have the same variance that has been perfectly balanced through the design of the sampling rule, the difference between the left-hand side and right-hand side of the inequality above will be the same for all treatments $a$. Then, using \texttt{SHRVar} for elimination would not incur bias towards any treatment and would still produce a good performance shown in Fig.~\ref{fig:err}. However, if the variances are not balanced, the underestimation gap would be different for each treatment, and therefore, an implicit bias would occur if we eliminate treatments based on $\min_i \hat{z}_s(a,i)$, which is harmful and increases the probability of error. The confidence-based elimination rule in \texttt{SHRVar-c} will alleviate this underestimation issue since the confidence intervals will still be valid even if the variances are not well-balanced, and thus the confidence level $\delta_s(a)$ is not biased. However, the computation of $\delta_s(a)$ involves multiple treatments and may incur additional variance in the elimination process, which explains the slightly worse empirical performance compared to \texttt{SHRVar} in Fig.~\ref{fig:err}. One could also use methods such as the double estimator~\citep{hasselt10doubleQ} to mitigate this issue, but this will also add additional variance to each estimator since only half of the samples are used to identify $a^*$. 

\subsection{Theoretical Guarantee}
We present the probability of error upper bound guarantee for \texttt{SHRVar-c} in algorithm~\ref{alg:shrvarc} analogous to Theorem~\ref{thm:estimate} in this section. We also assume $A$ is a power of $2$. Recall that for an arbitrary set $\gS$ of treatments, the heterogeneity of relative variance $\kappa_{\gS, a, i}$ is defined as follows:
\begin{align*}
        \kappa_{\gS, a, i} = \frac{\frac{\rho_{a,i}^2}{\max_i\rho_{a,i}^2} \rho_{\gS, \Sigma}  + \frac{\lambda_{a,i}^2}{\max_{a \in \gS, i} \lambda^2_{a,i} }\lambda_{\gS,\Sigma}}{\rho_{\gS, \Sigma} + \lambda_{\gS,\Sigma}}.
\end{align*}
We additionally define another effective gap $\tilde{D}_{\gS, a}$ similar to the gap $D_{\gS, a}$ used in Theorem.~\ref{thm:estimate} as follows:
\begin{align*}
    \tilde{D}_{\gS, a}^2 =\min_{i \in [M]} \max_{j\in[M]} \min\left\{\frac{[z_{a^*,i} - z_{a,j}]_+^2}{(\kappa_{\gS, a,j} + \kappa_{\gS, a^*,i})^2}, A_{\gS, a, i,j} \right\},
\end{align*}
where the additional term $A_{\gS, a, i, j} = +\infty$ if $\kappa_{\gS, a, j} \leq \kappa_{\gS, a, i}$, and is expressed as follows:
\begin{align*}
    A_{\gS, a, i, j} = \left\{
    \begin{array}{cc}
        +\infty & \kappa_{\gS, a, j}, \leq \kappa_{\gS, a, i} \\
        \frac{[z_{a^*,i} - z_{a,j}]_+^2}{\left( \kappa_{\gS, a, j} - \kappa_{\gS, a^*, i} \right)^2}+ \Delta_{\min}^2 - \frac{8A\log_2^2 A}{T}, & \kappa_{\gS, a, j} > \kappa_{\gS, a, i}
    \end{array}
    \right..
\end{align*}
Finally, for any set of treatments $\gS$, let $\tilde{\gS}'_{\ervc}$ be its subset excluding $1/4$ of the treatments with smallest $\tilde{D}_{\gS, a}$ similar to the definition of $\gS'_{\ervc}$ in Theorem.~\ref{thm:estimate}. The $\tilde{H}_3$ complexity measure used to characterize the performance of \texttt{SHRVar-c} is defined as:
\begin{align*}
    \tilde{H}_3^{-1} = \min_{\gS: a^*\in \gS} \frac{\min_{a\in \tilde{\gS}'_{\ervc}}\tilde{D}^2_{\gS, a}}{ \left( \rho_{\gS, \Sigma} + \lambda_{\gS,\Sigma} \right)^2}.
\end{align*}
Theorem~\ref{thm:conf} states the probability of mistake cab be bounded as:
\begin{align*}
        \prob\left( \hat{a}\neq a^* \right) \leq 6 M \log_2 A \cdot \exp\left( -\frac{T}{2 \tilde{H}_3 \log_2 A} \right).
    \end{align*}
\textbf{Proof Roadmap.} For Theorem~\ref{thm:conf}, we use the same proof framework as the proof of Theorem~\ref{thm:estimate} except that we need to analyze the probability that a sub-optimal treatment has lower confidence $\delta_s(a)$ than the best treatment. This is more challenging since it also involves a third treatment, i.e., the treatment that has the largest lower confidence bound. Our solution is to couple the confidence-based elimination rule with the simple elimination rule based on the minimum $z$ value used in \texttt{SHRVar}, to only analyze the event where the concentration of the minimum $z$ value estimate does not hold. This can only happen when the confidence $\delta_s(a^*)$ for the best treatment is small, meaning a large deviation happens between the best treatment and the treatment with the largest LCB. From this coupling argument, we will be able to reach the instance-dependent exponent $\tilde{D}_{\gA_s, a}$ which is almost the same as $D_{\gA_s, a}$ in Theorem~\ref{thm:estimate}. The complete proof is in the appendix section~\ref{sec:proof-conf}

\textbf{Effective Gap.} The effective gap notion $\tilde{D}_{\gS, a}$ in Theorem.~\ref{thm:conf} is slightly different from $D_{\gS, a}^2$ in Theorem.~\ref{thm:estimate} due to the additional $A_{\gS, a, i, j}$ term in the case where $\kappa_{\gA_s, a, j} > \kappa_{\gA_s, a^*, i}$. This means the confidence bonus $b_s^\delta(a, j)$ for the sub-optimal treatment $a$ is larger than the best treatment., which will make the best treatment ``vulnerable'' to elimination as it is easier to increase $\delta$ to create a separation of confidence intervals for the best treatment even if it has a larger empirical $z$ value due to less variance. The additional term $A_{\gS, a, i, j}$ accounts for the probability that this case happens. When $T$ is as large as $\tilde{\Omega}(A \Delta_{\min}^{-2})$, the two effective gaps will be the same, since:
\begin{align*}
    A_{\gS, a, i, j} \geq \frac{[z_{a^*,i} - z_{a,j}]_+^2}{\left( \kappa_{\gS, a, j} - \kappa_{\gS, a^*, i} \right)^2} \geq \frac{[z_{a^*,i} - z_{a,j}]_+^2}{\left( \kappa_{\gS, a, j} + \kappa_{\gS, a^*, i} \right)^2},
\end{align*}
and thus will be vacuous when the minimization is taken. This means when the exploration time horizon $T$ is small, \texttt{SHRVar-c} may have a slightly larger probability of error compared to \texttt{SHRVar}. As $T$ increases, the difference will become smaller and smaller, and finally, the two algorithms will have the same performance when $T$ is large enough. This theoretical insight coincides with our numerical experiment in Fig.~\ref{fig:err} as the gap between \texttt{SHRVar} and \texttt{SHRVar-c} becomes smaller as the time horizon increases. Therefore, in the large $T$ regime, \texttt{SHRVar-c} will be equivalent to \texttt{SHRVar} and enjoys all the characteristics of \texttt{SHRVar}. 

\section{Discussions}
In this section, we provide further discussions of some aspects of our paper that may provide insights to future research and better engage our findings with the existing literature.

\subsection{Variance Equalization Compared to Neyman Allocation}\label{sec:appendix-neyman}

In the best arm identification problems with a fixed budget, the essential goal is to design a sampling strategy that optimally reduces the uncertainty of the estimators and results in the minimum probability of error. The variance equalization design principle is widely used in the best arm identification literature~\citep{weltz23variance,lalitha23variance,antos10variance}, even though other principles are also studied. In this section, we discuss the reason and advantage of choosing the variance equalization sampling rule in this paper to design \texttt{SHRVar}, and compare it to the well-known Neyman allocation.

\textbf{Neyman Allocation.} The Neyman allocation arises from estimation problems in stratified sampling tasks~\citep{neyman1992allocation}. The essential idea of Neyman allocation is to assign the number of samples for each random variable proportional to its standard deviation. In the fixed budget best arm identification problem, the number of pulls for each arm under Neyman allocation would be proportional to the standard deviation of the reward, i.e., $N(a)\varpropto \sigma_a$, where $N(a)$ is the number of samples for each arm and $\sigma_a$ is the standard deviation. The Neyman allocation has been proven optimal in statistical estimation~\citep{glynn2004neymanallocation} or evaluation tasks~\citep{li2023neymanallocation}, but the performance in decision-making tasks, especially for the setting where there exists multiple arms, is elusive. In~\citep{kaufmann16bai}, it has been shown that for a limited setting, i.e., a two-armed fixed-budget best arm identification problem with Gaussian noise, Neyman allocation achieves the optimal asymptotic rate of error as follows and matches the lower bound.
\begin{align*}
    \lim_{T\to\infty} -\frac{1}{T} \log \prob\left( \hat{a} \neq a^* \right) = \frac{(\mu_1 - \mu_2)^2}{(\sigma_1 + \sigma_2)^2},\quad \text{for Neyman allocation}.
\end{align*}
where $\hat{a}$ is the recommended arm by the algorithm, $a^*$ is the best arm with highest expected reward, and $\mu_1$ and $\mu_2$ are the reward mean of both arms. But the performance for more general settings, such as multiple arms or other reward distributions, is not optimal as shown in our experiment in Sec~\ref{sec:appendix-experiment}.

\textbf{Variance Equalization Allocation.} Our sampling rule design principle is variance equalization, which has been efficiently used in best arm identification problems under both fixed budget~\citep{lalitha23variance} and fixed confidence settings~\citep{weltz23variance,antos10variance} with multiple arms. Specifically, the principle aims to design and balance the number of pulls for each arm, so that the variance of each estimator for decision-making will be the same. In the classic multiple-armed bandit problem, it will result in an allocation rule where the number of pulls $N(a)$ for each arm $a$ is proportional to the variance $\sigma_a^2$ of that arm. The sequential halving algorithm designed based on this rule is optimal up to a logarithmic factor in best arm identification with fixed confidence~\citep{weltz23variance}.

\textbf{Comparison.} We show that variance equalization allocation is more general and has a better performance over a larger set of problems compared to Neyman allocation, without losing much optimality in the two-armed Gaussian bandit problem. This justifies why we chose this rule to design \texttt{SHRVar}. We consider a multi-armed bandit model with $A+2$ arms. Suppose arm $1$ is the best arm with expected reward $\mu_1$, and all other arms have the same reward mean $\mu_2$. Both arms $1$ and $2$ have large reward variance $\Sigma^2$, and all other arms have small reward variance $\sigma^2$, and we assume $\sigma$ is much smaller than $\Sigma$. Then, with a large enough time horizon $T$, the bottleneck of the best arm identification task in this model would be the test between arm $1$ and arm $2$, and the exponent of error probability for each allocation rule would be as follows:
\begin{align*}
    \lim_{T\to\infty} -\frac{1}{T} \log \prob\left( \hat{a} \neq a^* \right) =& \frac{(\mu_1 - \mu_2)^2}{4\Sigma^2+2A\Sigma \sigma}, \quad \text{for Neyman allocation;}\\
    \lim_{T\to\infty} -\frac{1}{T} \log \prob\left( \hat{a} \neq a^* \right) =& \frac{(\mu_1 - \mu_2)^2}{4\Sigma^2+2A \sigma^2}, \quad \text{for variance equalization allocation.}
\end{align*}
Therefore, Neyman allocation has a strictly worse performance than variance equalization for $\sigma < \Sigma$, independent of the number of arms. As the number of arms $A$ becomes much larger, especially much larger than $\Sigma^2 / \sigma^2$, the denominator of Neyman allocation can be arbitrarily larger than variance equalization, i.e.,
\begin{align*}
    \lim_{A\to \infty} \lim_{T\to\infty} \frac{\log \prob\left( \hat{a} \neq a^* | \text{Neyman} \right)}{\log \prob\left( \hat{a} \neq a^* | \text{variance equalization}\right)} = \lim_{A\to \infty} \frac{4\Sigma^2+2A \Sigma\sigma}{4\Sigma^2+2A \sigma^2} = \frac{\Sigma}{\sigma}.
\end{align*}
Since $\Sigma$ can be arbitrarily larger than $\sigma$, the probability of error for Neyman allocation can be arbitrarily larger than that of variance equalization allocation. This is because variance-equalization assigns more samples to large variance arms, which is helpful when the bottleneck involves two arms with large variance. Therefore, variance-equalization is a more general approach that is near-optimal over a wider range of problems. This is also why it has been much more adopted in the best arm identification literature than Neyman allocation~\citep{weltz23variance,lalitha23variance}. On the other hand, in the specialized two-armed Gaussian bandit setting, the performance of variance equalization allocation is still optimal up to a factor of $2$, an absolute constant which cannot be arbitrarily worse, i.e.,
\begin{align*}
    \lim_{T\to\infty} -\frac{1}{T} \log \prob\left( \hat{a} \neq a^* \right) = \frac{(\mu_1 - \mu_2)^2}{2\left(\sigma_1^2 + \sigma_2^2\right)},\quad \text{for variance equalization allocation}.
\end{align*}
And we have:
\begin{align*}
    \frac{(\mu_1 - \mu_2)^2}{2\left(\sigma_1^2 + \sigma_2^2\right)} \geq \frac{1}{2} \frac{(\mu_1 - \mu_2)^2}{(\sigma_1 + \sigma_2)^2},
\end{align*}
where the left-hand side is half of the exponent for Neyman allocation. This also shows that variance equalization does not suffer much performance loss in the specialized setting and should still achieve near-optimal probability of error. Overall, variance equalization allocation achieves a better performance in our studied \texttt{M3AB} problem as shown in our additional experiments in Sec~\ref{sec:appendix-experiment}.

\subsection{Equivalence of Max-Min and Expectation Characterization of Best Treatment}\label{sec:appendix-equiv}
In this section, we illustrate and discuss the relation between the max-min formulation studied in our paper in~\eqref{eq:best-prob} and a more natural formulation that maximizes the expected probability of passing validation. Recall that $\gE_{\rv,i}$ is the event that the treatment $\hat{a}$ selected by the exploration phase beats the control in the validation A/B test. The two formulations can be represented as follows:
\begin{align*}
    \text{Max-Min Formulation:} \quad &a^*_1 = \argmax_{a\in \gA / \{0\}} \min_{i\in [M]} \prob\left( \gE_{\rv,i} \right);\\
    \text{Expected Formulation:} \quad &a^*_2 = \argmax_{a\in \gA / \{0\}} \prob\left( \bigcap_{i\in [M]}\gE_{\rv,i} \right);
\end{align*}
Let's consider the non-Bayesian validation with two treatments $\gA = \{0, 1, 2\}$ and two metrics $M=2$ with $\delta_i = 0.5$ for simplicity. Assume both treatments are better than the control for both metrics. The argument can easily be generalized to settings with multiple treatments and metrics. Suppose the recommended treatment to validate is a treatment $a$. Then, by proposition \ref{prop:val}, we have:
\begin{align*}
    a^*_1 = \argmax_{a\in \{1,2\}} \min_{i\in \{1,2\}} \left(\frac{\mu_{a,i} - \mu_{0,i}}{\sqrt{\sigma_{a,i}^2 + \sigma_{0,i}^2}} + \xi_{a,i}\right).
\end{align*}
Recall that from the proof of proposition \ref{prop:val}, we have:
\begin{align*}
    \gE_{\rv, i} = \left\{\hat{\mu}_{\rv}(a,i) - \hat{\mu}_{\rv}(0,i) \geq \Phi^{-1}(1 - \delta_i) \sqrt{\frac{\sigma_{a,i}^2 + \sigma_{0,i}^2 }{T_{\rv} / 2}} \right\}.
\end{align*}
The probability of failure can be characterized using the cumulative distribution function $\Phi(\cdot)$ of standard normal as follows:
\begin{align*}
    \prob\left( \gE_{\rv, i}^\complement \right)
    = \Phi\left(  - \frac{\mu_{a,i} - \mu_{0,i}}{\sqrt{\sigma_{a,i}^2 + \sigma_{0,i}^2 }} \sqrt{T_{\rm{v}}/2} \right).
\end{align*}
On the other hand, we consider the expected formulation as follows. First, by the independence among metrics, we have:
\begin{align*}
    \prob\left( \bigcap_{i\in [M]}\gE_{\rv,i} \right) =& \prod_{i\in [M]} \prob\left( \gE_{\rv,i} \right) = \prod_{i\in [M]} \left( 1- \prob\left( \gE_{\rv,i}^{\complement} \right) \right)\\
    =&\left( 1- \Phi\left(  - \frac{\mu_{a,1} - \mu_{0,1}}{\sqrt{\sigma_{a,1}^2 + \sigma_{0,1}^2 }} \sqrt{T_{\rm{v}}/2} \right) \right) \left( 1- \Phi\left(  - \frac{\mu_{a,2} - \mu_{0,2}}{\sqrt{\sigma_{a,2}^2 + \sigma_{0,2}^2 }} \sqrt{T_{\rm{v}}/2} \right) \right)\\
    =& 1 - \underbrace{\Phi\left(  - \frac{\mu_{a,1} - \mu_{0,1}}{\sqrt{\sigma_{a,1}^2 + \sigma_{0,1}^2 }} \sqrt{T_{\rm{v}}/2} \right)}_{A_1(a)} - \underbrace{\Phi\left(  - \frac{\mu_{a,2} - \mu_{0,2}}{\sqrt{\sigma_{a,2}^2 + \sigma_{0,2}^2 }} \sqrt{T_{\rm{v}}/2} \right)}_{A_2(a)}\\
    &+ \underbrace{\Phi\left(  - \frac{\mu_{a,1} - \mu_{0,1}}{\sqrt{\sigma_{a,1}^2 + \sigma_{0,1}^2 }} \sqrt{T_{\rm{v}}/2} \right) \Phi\left(  - \frac{\mu_{a,2} - \mu_{0,2}}{\sqrt{\sigma_{a,2}^2 + \sigma_{0,2}^2 }} \sqrt{T_{\rm{v}}/2} \right)}_{A_{1,2}(a)}.
\end{align*}
So we have:
\begin{align*}
    a_2^* = \argmin_{a\in\{1,2\}} A_1(a) + A_2(a) - A_{1,2}(a).
\end{align*}
Since both treatments are better than the control, when $T_{\rv}$ is large, all three terms approach $0$. Considering the shape of the tail of the normal CDF, i.e., $\Phi(-\alpha \sqrt{T}) = \Theta(\mathrm{poly}(T)\exp(-\alpha^2 T^2))$, it is not difficult to see that all three terms $A_1(a)$, $A_2(a)$, and $A_{1,2}(a)$ decreases to $0$ exponentially fast. Therefore, the exponent $\alpha$ will determine the largest term among the three. Since all terms are decreasing exponentially, the largest term will dominate the choice of the $\argmin$ since the difference in other terms among arms will be negligible compared to the largest term. Therefore, we can transform the problem into the following when $T_{\rv}$ is large enough:
\begin{align*}
    a_2^* =& \argmin_{a\in\{1,2\}} A_1(a) + A_2(a) - A_{1,2}(a)\\
    =& \argmin_{a\in\{1,2\}} \max\left\{ A_1(a), A_2(a), A_{1,2}(a) \right\}\\
    =& \argmax_{a\in\{1,2\}} \min\left\{ \frac{\mu_{a,1} - \mu_{0,1}}{\sqrt{\sigma_{a,1}^2 + \sigma_{0,1}^2 }}, \frac{\mu_{a,2} - \mu_{0,2}}{\sqrt{\sigma_{a,2}^2 + \sigma_{0,2}^2 }}, \frac{\mu_{a,1} - \mu_{0,1}}{\sqrt{\sigma_{a,1}^2 + \sigma_{0,1}^2 }}+ \frac{\mu_{a,2} - \mu_{0,2}}{\sqrt{\sigma_{a,2}^2 + \sigma_{0,2}^2 }} \right\}\\
    =& \argmax_{a\in\{1,2\}} \min\left\{ \frac{\mu_{a,1} - \mu_{0,1}}{\sqrt{\sigma_{a,1}^2 + \sigma_{0,1}^2 }}, \frac{\mu_{a,2} - \mu_{0,2}}{\sqrt{\sigma_{a,2}^2 + \sigma_{0,2}^2 }} \right\}.
\end{align*}
Comparing this to the definition of $a_1^*$, we can easily conclude that $a_1^* = a_2^*$ in the large $T_{\rv}$ regime, and the two problem formulations are equivalent.

\subsection{Relaxing Assumptions}
We discuss intuitions and methods to relax our assumptions in the main text, such as known variances, Gaussianity, and independence. We remark that these assumptions are natural in the multi-armed bandit literature when developing algorithms for new models, such as this paper.

\textbf{Non-Existence of Uniformly Better Treatment.} It is possible that no treatment is better than control in all metrics simultaneously, and it would be ideal to detect this circumstance as quickly as possible in the exploration phase. A counterfactual evaluation over the samples from the recommended treatment and the control could be performed after the exploration phase to decide the confidence of whether the recommended treatment is better than the control. For each metric, the probability of failure of A/B validation could be calculated as characterized in our proof for Proposition~\ref{prop:val}. If the overall failure probability is too high, we could decide not to move to A/B validation and declare the non-existence of a universally better treatment. Even if we indeed move to A/B in this setting, the validation itself serves as a second protective mechanism to prevent a worse-than-control arm slip through. A numerical experiment on this null setting is also conducted in Appendix~\ref{sec:appendix-experiment}.

\textbf{Unknown Treatment Variance.} Even though knowing the variance of a treatment is a common assumption in developing bandit algorithms from theoretical causes, especially under novel problems with unique challenges, such as our paper, it is indeed in practice that the true variances of treatments are largely unknown. An adaptation we use is to use a small portion of the budget to estimate the variances beforehand and treat the estimated variances as the true variances. We are aware of this challenge, so we numerically tested \texttt{SHRVar} with this adaptation in Fig.~\ref{fig:unknown} of the appendix, which shows almost no performance loss compared to the known variance case. We believe the variance estimation error is likely to be secondary in most applications, and we also find this method works well in real-world applications. In practical use cases, we observe that different metrics, say revenue versus click rate, can have very different variances, where revenue is subject to a large variance while click rate has a much lower variance. For a given metric, different treatments may exhibit varying variances, but these differences are typically of a much smaller magnitude compared to the ones observed across metrics. Therefore, in our work, we prioritize the multi-objective challenge.

\textbf{Beyond Sub-Gaussian Rewards.} Although the Gaussian distribution is a commonly adopted assumption for developing bandit algorithms in new models, such as the multi-metric problem with control in our paper, it is exciting to see if extension beyond these assumptions is possible. We believe it should not pose too much difficulty to generalize \texttt{SHRVar} to other distributions based on a number of reasons from intuition. For example, the $z$-value objective in Proposition~\ref{prop:val} is derived by characterizing the failure probability of the validation A/B test, which is shown to have a monotonic relation with the $z$-value. For other distributions, such as the exponential distribution, one could also characterize this failure probability. The $z$-value objective would stay the same as long as the failure probability has a similar monotonic relation with the $z$-value. However, the tail characteristic of the reward distribution will affect the error bound in Theorem~\ref{thm:estimate}, because our proof relies on Hoeffding’s concentration, which is not valid beyond a sub-Gaussian distribution. But we should be able to use other concentration methods to prove a similar error bound, for example, using Bernstein's inequality for sub-exponential distributions. 

\textbf{Incorrect Prior.} A correct and good prior for Bayesian validation would possibly speed up the validation process, i.e., using fewer samples to produce the required significance. On the other hand, an incorrect prior could also affect the performance of Bayesian validation in our algorithm, similar to all other Bayesian algorithms. How to choose the prior is a rich literature that is complementary to this work. In our paper, if the prior of a metric is picked too large or too small, it would change the $z$-value of each arm for that metric, and affect the choice of the best treatment since all metrics are related via the $\min$ operator in Proposition~\ref{prop:val}. However, this should only be problematic if the validation time horizon $T_v$ is too small, since the prior variance is weighted by $1/T_v$ in our $z$-value object.

\textbf{Metric Correlation.} It is an interesting topic to extend our paper to settings with metric correlation, which would be our future work. In most A/B tests with multiple metrics, each metric is evaluated individually, so we believe it still makes sense to follow the objective in Eq.~\ref{eq:best-prob} to optimize the minimum passing probability over metrics. Then, the problem can still be transformed into the $z$-value max-min problem in Proposition~\ref{prop:val}, independent of the existence of correlation. If the correlation is known, it will play an important role in our sampling rule to estimate the $z$-values. We could follow the same framework of this paper to analyze the covariance matrix of the $z$-estimators and perform an optimal design based on it. Leveraging this correlation would potentially improve the algorithm’s performance, resulting in a better error bound, since assuming independence is essentially assuming the worst case. But on the other hand, knowing the metric correlation is as difficult as knowing the variance of each treatment in practice, if not more difficult.

\subsection{Other Discussions}\label{sec:appendix:discuss}

\textbf{Lower Bounds for Fixed-Budget Problems.} 
Developing tight lower bounds for the \texttt{M3AB} problem is a challenging but meaningful problem, which we aim to address as our future direction. It is difficult in the sense that, to our best knowledge, no instance-dependent error probability lower bound has been developed in the literature beyond a two-armed Gaussian setting. Even for classic best arm identification with multiple arms and homogeneous variance, such a tight matching lower bound has not been developed, and our studied \texttt{M3AB} problem is much more complex, involving multiple treatments, multiple metrics, and a control arm. Some researchers also question whether an achievable lower bound exists in the fixed budget setting~\citep{degenne19goodarm}, where the authors prove that if there exists a lower bound for the fixed-budget best arm identification problem, there is no single algorithm that can achieve it on every problem instance. And our problem is even more general. On the other hand, the complexity of our proposed \texttt{SHRVar} is a generalization of \texttt{SH}, one of the best fixed-budget algorithms in single-metric unit-variance problems, and therefore inherits the advantages of \texttt{SH}. 

\textbf{Relative Variance and Validation.} One of the major contributions of our paper is to formally introduce the validation phase and the control into the best arm identification framework with multiple metrics for consideration. 
The problem studied in this paper has unique characteristics which makes classic algorithms such as \texttt{SH} not applicable, e.g., (i) the control is special and needs to be considered separately in exploration; (ii) the objective shifts from identifying the treatment that maximizes the mean reward to the treatment that maximizes the probability of passing validation, which reflects a mean-variance trade-off; (iii) multi-metric. Some of the aspects have been studied separately (while others, such as the control, are overlooked), but the combination received little attention. The algorithm under this general setting requires a systematic design instead of collecting separate solutions to each aspect. Our starting point is to maximize the probability of passing validation. We systematically and jointly analyzed the exploration and validation phases and reduced them to identify the treatment with the best mean-variance trade-off, i.e., the largest z value. By studying the z-value estimate with the variance equalization allocation principle, the relative variance naturally arises, and we combine the novel relative-variance-based allocation with a sequential halving framework to propose \texttt{SHRVar}. This algorithm is provably efficient and strictly generalizes well-known algorithms such as sequential halving in classic BAI. It also connects to real-world applications, as we cannot control the downstream validation, where the experimentation time and passing criteria are set agnostically.

\textbf{Importance Differences of Metrics.} It is common in many real-world applications, where some metrics (say revenue) are “must-haves,” and some (say engagement rate or downstream impact) are “nice-to-haves.” If some metrics are less important, one could adjust the parameters of the A/B test related to this metric to make it easier to pass, for example, increase the confidence level $\delta_i$ or decrease the required posterior $p_i$ for this metric. If some metrics are more important, one could also adjust the same parameter to make it harder to pass. There may be more than one treatment that surpasses all the constraints. Our current algorithm will pick the best one that is most distinguishable from all constraints, which may not necessarily be the one with the highest value in the “must-have” metric. This is a valuable extension for our future work.

\textbf{Technical Challenges and Contributions.} The high-level framework of designing \texttt{SHRVar} is rooted in the adaptive experimental design literature, where the sampling rule strives to equalize the variance of the estimators for different treatments, and the proof of Theorem~\ref{thm:estimate} follows a general framework presented~\cite{karnin13bai} to prove the theoretical performance of \texttt{SH}. However, many technical subtleties need to be addressed before applying this framework: 
\begin{enumerate}
    \item Performing BAI based on our objective $\min_i z_{a,i}$ has two unique challenges, i.e., not knowing the bottleneck metric $i_{a}^*$ makes it difficult to analyze the variance, and the estimator of each $z_{a,i}$ may also be correlated due to the common control. Therefore, we seek surrogates to decouple the estimators by equalizing a variance upper bound and extending from single-metric to multi-metric settings as discussed in Section~\ref{sec:shrvar}. 
    \item In the proof of Theorem 5.1, we generalize the idea from the $H_2$ complexity for \texttt{SH}~\citep{karnin13bai} for heterogeneous and multi-metric cases to form the notion of effective gap, relative variance, and set $S_c'$ to capture the trade-off between gap and variance of each treatment through the complexity term $H_3$. This step is not direct due to the heterogeneous nature of the problem. As a comparison, for \texttt{SHVar} proposed in~\citet{lalitha23variance}, which also studies BAI in heterogeneous variance cases, only uses the minimum reward gap and the sum of variances for all arms, and thus fails to capture the reward-variance trade-off in each arm. Moreover, due to the correlation between the estimators of treatments and metrics, we analyze the concentration events more carefully instead of taking a direct union bound over arms and metrics (which would result in additional dependence on $M$ and $A$ outside the exponential upper bound).
\end{enumerate}

\section{Proof of Proposition~\ref{prop:val}}\label{sec:appendix-proof-prop}
Recall the explicit forms of the posterior mean and variances $\hat{\Delta}_{\rv, i}$ and $\hat{\sigma}^2_{\rv, i}$ in the Bayesian validation as follows. The explicit expression can be derived easily from~\cite[equation 2]{kveton22bayesianfixbudget}
\begin{align*}
    \hat{\sigma}_{\rv, i}^2 = \left(\frac{ T_{\rm{v}} }{2\left(\sigma_{\hat{a},i}^2 + \sigma_{0,i}^2\right)} + \frac{1}{\tau_i^2} \right)^{-1}, \quad 
    \hat{\Delta}_{\rv, i} = \frac{T_{\rv}\hat{\sigma}_{\rv, i}^2}{2\left( \sigma_{\hat{a},i}^2 + \sigma_{0,i}^2 \right)} \left( \hat{\mu}_{\rv}(\hat{a},i) - \hat{\mu}_{\rv}(0,i)  \right) .
\end{align*}
Then, we prove the proposition. In both Bayesian and non-Bayesian validation, we fix a reward metric $i$ and characterize the probability that treatment $a$ will pass the validation A/B test against control, i.e., $\gE_{\rv, i}^\complement$ when $\hat{a}=a$. We first start with the non-Bayesian validation. Recall that passing the validation means the difference in the empirical mean should be large enough, i.e.,
\begin{align*}
    \gE_{\rv, i} = \left\{\hat{\mu}_{\rv}(a,i) - \hat{\mu}_{\rv}(0,i) \geq \Phi^{-1}(1 - \delta_i) \sqrt{\frac{\sigma_{a,i}^2 + \sigma_{0,i}^2 }{T_{\rv} / 2}} \right\}.
\end{align*}
Since the reward follows a normal distribution, the empirical estimate will also be normal, and therefore, the probability of failure can be characterized using the cumulative distribution function $\Phi(\cdot)$ of standard normal as follows:
\begin{align*}
    \prob\left( \gE_{\rv, i}^\complement \right) =& \prob\left( \hat{\mu}_{\rv}(a,i) - \hat{\mu}_{\rv}(0,i) \leq \Phi^{-1}(1 - \delta_i) \sqrt{\frac{\sigma_{a,i}^2 + \sigma_{0,i}^2 }{T_{\rv} / 2}} \right)\\
    =& \Phi\left( \frac{\Phi^{-1}(1 - \delta_i) \sqrt{\frac{\sigma_{a,i}^2 + \sigma_{0,i}^2 }{T_{\rv} / 2}} - (\mu_{a,i} - \mu_{0,i})}{\sqrt{\mathrm{Var}(\hat{\mu}_{\rv}(a,i) - \hat{\mu}_{\rv}(0,i))}} \right)\\
    =& \Phi\left(  \Phi^{-1}(1 - \delta_i) - \frac{\mu_{a,i} - \mu_{0,i}}{\sqrt{\sigma_{a,i}^2 + \sigma_{0,i}^2 }} \sqrt{T_{\rm{v}}/2} \right).
\end{align*}
By the monotonicity of $\Phi(\cdot)$, the treatment that minimizes the probability of failure should:
\begin{align*}
    \argmin_{a\in \gA / \{0\}} \max_i \prob\left( \gE_{\rv, i}^\complement \right) 
    =& \argmin_{a\in \gA / \{0\}} \max_i \Phi\left(  \Phi^{-1}(1 - \delta_i) - \frac{\mu_{a,i} - \mu_{0,i}}{\sqrt{\sigma_{a,i}^2 + \sigma_{0,i}^2 }} \sqrt{T_{\rm{v}} / 2} \right)\\
    =& \argmin_{a\in \gA / \{0\}} \max_i \left(\Phi^{-1}(1 - \delta_i) - \frac{\mu_{a,i} - \mu_{0,i}}{\sqrt{\sigma_{a,i}^2 + \sigma_{0,i}^2 }} \sqrt{T_{\rm{v}} / 2}\right)\\
    =& \argmax_{a\in \gA / \{0\}} \min_i \left(\frac{\mu_{a,i} - \mu_{0,i}}{\sqrt{\sigma_{a,i}^2 + \sigma_{0,i}^2 }} + \frac{\Phi^{-1}(\delta_i)}{ \sqrt{T_{\rm{v}} / 2}}\right).
\end{align*}
where the last step uses the anti-symmetry of $\Phi(\cdot)$, i.e., $\Phi(x) = 1-\Phi(-x)$. Therefore, the validation constant in the non-Bayesian test is:
\begin{align*}
    \xi_{a,i} = \frac{\Phi^{-1}(\delta_i)}{ \sqrt{T_{\rm{v}} / 2}}.
\end{align*}
Next, we derive the validation constant $\xi_{a,i}$ in the Bayesian validation. We also fix a reward metric $i$ and characterize the probability that treatment $a$ will pass the validation A/B test against control, i.e., $\gE_{\rv, i}^\complement$ when $\hat{a}=a$. Recall that the criteria for passing the validation are that the posterior probability of $\{\mu_{a,i}>\mu_{0,i}\}$ is larger than a given threshold $q_i$. Since the posterior of $(\mu_{a,i} - \mu_{0,i})$ is a normal distribution, i.e., $\nu_{\rv, i} = \gN(\hat{\Delta}_{\rv, i}, \hat{\sigma}^2_{\rv, i})$, this is equivalent to that the cumulative distribution function of $\nu_{\rv, i}$, evaluated at $0$, is smaller than $1 - q_i$, which is also equivalent to the following event:
\begin{align*}
    \left\{\Phi\left( \frac{-\hat{\Delta}_{\rv, i}}{\hat{\sigma}_{\rv, i}}\right) \leq 1-q_i \right\} = \left\{ \frac{-\hat{\Delta}_{\rv, i}}{\hat{\sigma}_{\rv, i}} \leq \Phi^{-1}(1-q_i) \right\}.
\end{align*}
where the equality is from the monotonicity of $\Phi(\cdot)$. Substitute the expression of the posterior mean $\hat{\Delta}_{\rv, i}$ and variance $\hat{\sigma}^2_{\rv, i}$, we have:
\begin{align*}
    \gE_{\rv, i} =& \left\{ \frac{\frac{T_{\rv}\hat{\sigma}_{\rv, i}^2}{2\left( \sigma_{a,i}^2 + \sigma_{0,i}^2 \right)}\left( \hat{\mu}_{\rv}(a,i) - \hat{\mu}_{\rv}(0,i)  \right)}{\hat{\sigma}_{\rv, i}} \geq - \Phi^{-1}(1-q_i)\right\}\\
    =& \left\{  \hat{\mu}_{\rv}(a,i) - \hat{\mu}_{\rv}(0,i)  \geq \frac{2\left( \sigma_{a,i}^2 + \sigma_{0,i}^2 \right)}{T_{\rv}\hat{\sigma}_{\rv, i}} \Phi^{-1}(q_i)\right\}.
\end{align*}
Similarly, by the normal nature of the empirical estimate of reward, we can express the probability of failure in validation as follows:
\begin{align*}
    \prob\left( \gE_{\rv, i} ^\complement \right) =& \prob\left( \hat{\mu}_{\rv}(a,i) - \hat{\mu}_{\rv}(0,i)  
    \leq \frac{2\left( \sigma_{a,i}^2 + \sigma_{0,i}^2 \right)}{T_{\rv}\hat{\sigma}_{\rv, i}} \Phi^{-1}(q_i) \right)\\
    = & \Phi\left( \frac{\Phi^{-1}(q_i)\frac{2\left( \sigma_{a,i}^2 + \sigma_{0,i}^2 \right)}{T_{\rv}\hat{\sigma}_{\rv, i}} - (\mu_{a,i} - \mu_{0,i})}{\sqrt{\mathrm{Var}(\hat{\mu}_{\rv}(a,i) - \hat{\mu}_{\rv}(0,i))}} \right)\\
    = & \Phi\left( \frac{\Phi^{-1}(q_i)\frac{2\left( \sigma_{a,i}^2 + \sigma_{0,i}^2 \right)}{T_{\rv}\hat{\sigma}_{\rv, i}} - (\mu_{a,i} - \mu_{0,i})}{\sqrt{\frac{\sigma_{a,i}^2 + \sigma_{0,i}^2}{T_{\rm{v}} / 2}}} \right)\\
    =& \Phi\left(  \Phi^{-1}(q_i) \sqrt{\frac{2}{T_{\rm{v}}} \frac{\sigma_{a,i}^2 + \sigma_{0,i}^2}{\hat{\sigma}^2_{\rv, i}}} - \frac{\mu_{a,i} - \mu_{0,i}}{\sqrt{\sigma_{a,i}^2 + \sigma_{0,i}^2 }} \sqrt{T_{\rm{v}} / 2} \right).
\end{align*}
where we substitute the expression of $\hat{\sigma}^2_{\rv, i}$ as follows:
\begin{align*}
    \hat{\sigma}_{\rv, i}^2 = \left(\frac{ T_{\rm{v}} }{2\left(\sigma_{\hat{a},i}^2 + \sigma_{0,i}^2\right)} + \frac{1}{\tau_i^2} \right)^{-1},
\end{align*}
And then we have:
\begin{align*}
    \prob\left( \gE_{\rv, i} ^\complement \right) =& \Phi\left(  \Phi^{-1}(q_i) \sqrt{ 1 + \frac{2}{T_{\rm{v}}} \frac{\sigma_{a,i}^2 + \sigma_{0,i}^2}{\tau_i^2}} - \frac{\mu_{a,i} - \mu_{0,i}}{\sqrt{\sigma_{a,i}^2 + \sigma_{0,i}^2 }} \sqrt{T_{\rm{v}} / 2} \right).
\end{align*}
Therefore, similarly by the monotonicity and anti-symmetry of $\Phi(\cdot)$, our objective becomes:
\begin{align*}
    \argmin_{a\in \gA / \{0\}} \max_i \prob\left( \gE_{\rv, i}^\complement \right)  =  \argmax_{a\in \gA / \{0\}} \min_i \left(\frac{\mu_{a,i} - \mu_{0,i}}{\sqrt{\sigma_{a,i}^2 + \sigma_{0,i}^2 }} + \frac{\Phi^{-1}(1 - q_i)}{ \sqrt{T_{\rm{v}} / 2}}\sqrt{ 1 + \frac{2}{T_{\rm{v}}} \frac{\sigma_{a,i}^2 + \sigma_{0,i}^2}{\tau_i^2}}\right), 
\end{align*}
where the validation constant in the Bayesian test is:
\begin{align*}
    \xi_{a,i} = \frac{\Phi^{-1}(1-q_i)}{\sqrt{T_{\rv} / 2}} \sqrt{1 + \frac{\sigma_{a,i}^2 + \sigma_{0,i}^2}{\tau_i^2 (T_v / 2)}}.
\end{align*}

\section{Proof of Theorem~\ref{thm:estimate}}\label{sec:proof-estimate}
Suppose at stage $s$ and fix the active treatment set $\gA_s$, we first analyze the probability that the best treatment is eliminated. This can only happen if there exists at least $|\gA_s|$ arms with a larger minimum empirical $z$ value than the best treatment, i.e.,
\begin{align*}
    \left\{\sum_{a\in \gA_s} \mathbbm{1}_{\min_i \hat{z}_s(a,i) > \min_j \hat{z}_s(a^*,j)} \geq \frac{|\gA_s|}{2} \right\}.
\end{align*}
Therefore, it is essential to study the event $\{\min_i \hat{z}_s(a,i) > \min_j \hat{z}_s(a^*,j)\}$ for all sub-optimal arms. We first present the following lemma, which bounds the probability that the minimum z-value estimate of a suboptimal treatment $a$ would be larger than the best treatment.

\begin{lemma}
    Fix a stage $s$ and the active treatment set $\gA_s$, we have:
    \begin{align*}
        &\prob\Big(\min_i \hat{z}_s(a,i) > \min_j \hat{z}_s(a^*,j)\Big)\\
        \leq & 2 M \exp\Bigg( -\frac{T}{2 \left( \rho_{\gA_s,\Sigma} + \lambda_{\gA_s,\Sigma} \right)^2\log_2 A} 
        \cdot \min_{j\in [M]} \max_{z\leq z_{a^*,j}} \min\Bigg\{\frac{[z_{a^*,j} - z]_+^2}{\kappa^2_{\gA_s, a^*,j}}, \sum_{i:z_{a,i}< z_{a^*,j}}\frac{[z - z_{a,i}]_+^2}{ \kappa^2_{\gA_s, a,i}} \Bigg\} \Bigg),
    \end{align*}
    where $(\rho_{\gA_s, \Sigma}, \lambda_{\gA_s,\Sigma})$ measures the relative variance of the active treatment set $\gA_s$, i.e.,
    \begin{align*}
    \rho_{\gA_s,\Sigma}^2 =  \sum_{a\in \mathcal{A}_s} \max_i\rho_{a,i}^2, \quad \lambda_{\gA_s,\Sigma}^2 = \max_{a \in \gA_s} \max_i \lambda^2_{a,i};
\end{align*}
    and $\kappa^2_{\gA_s, a,i}$ measures the heterogeneity of relative variance of arm $a$ and metric $i$, i.e.,
    \begin{align*}
        \kappa^2_{\gA_s, a,i} = \frac{\rho_{a,i}^2}{\max_i\rho_{a,i}^2} \frac{\rho_{\gA_s, \Sigma} }{\rho_{\gA_s, \Sigma} + \lambda_{\gA_s,\Sigma}} + \frac{\lambda_{a,i}^2}{\max_{a \in \gA_s} \max_i \lambda^2_{a,i} }\frac{\lambda_{\gA_s,\Sigma}}{\rho_{\gA_s, \Sigma} + \lambda_{\gA_s,\Sigma}}.
    \end{align*}
\end{lemma}

\begin{proof}
We first analyze the event and separate the reward metrics, i.e.,
\begin{align*}
    \{\min_i \hat{z}_s(a,i) > \min_j \hat{z}_s(a^*,j)\} =& \bigcup_{j\in [M]}\left\{ \min_i \hat{z}_s(a,i) > \hat{z}_s(a^*,j)\right\}
    = \bigcup_{j\in [M]} \bigcap_{i\in[M]} \left\{ \hat{z}_s(a,i) > \hat{z}_s(a^*,j)\right\}.
\end{align*}
Now, we analyze the intersection of events for a fixed $j$. First, we ignore the reward metrics $i$ where $z_{a,i} \geq z_{a^*,j}$, and upper bound the event as:
\begin{align*}
    \bigcap_{i\in[M]} \left\{ \hat{z}_s(a,i) > \hat{z}_s(a^*,j)\right\} \subset \bigcap_{i: z_{a,i} < z_{a^*,j}} \left\{ \hat{z}_s(a,i) > \hat{z}_s(a^*,j)\right\}.
\end{align*}
For an arbitrary constant $z \leq z_{a^*,j}$, we could separate the two estimates as follows:
\begin{align*}
    \left\{ \hat{z}_s(a,i) > \hat{z}_s(a^*,j)\right\} 
    = & \left\{ \hat{z}_s(a,i) > \hat{z}_s(a^*,j), \hat{z}_s(a^*,j) < z\right\} \cup \left\{ \hat{z}_s(a,i) > \hat{z}_s(a^*,j), \hat{z}_s(a^*,j) \geq z\right\}\\
    \subset & \left\{ \hat{z}_s(a^*,j) < z\right\} \cup \left\{ \hat{z}_s(a,i) > z\right\}.
\end{align*}
So we can upper bound the intersection events as follows:
\begin{align*}
    \bigcap_{{i: z_{a,i} < z_{a^*,j}}} \left\{ \hat{z}_s(a,i) > \hat{z}_s(a^*,j)\right\} 
    \subset & \bigcap_{{i: z_{a,i} < z_{a^*,j}}} \left(\left\{ \hat{z}_s(a^*,j) < z\right\} \cup \left\{ \hat{z}_s(a,i) > z\right\}\right)\\
    =& \left\{ \hat{z}_s(a^*,j) < z\right\} \cup \left(\bigcap_{i: z_{a,i} < z_{a^*,j}}\left\{ \hat{z}_s(a,i) > z\right\}\right).
\end{align*}
Then, we evaluate the probability of this event using a union bound as follows:
\begin{align*}
    \prob\left( \bigcap_{i: z_{a,i} < z_{a^*,j}} \left\{ \hat{z}_s(a,i) > \hat{z}_s(a^*,j)\right\} \right) 
    \leq & \prob\left(  \hat{z}_s(a^*,j) < z\right) +  \prob\left( \bigcap_{i: z_{a,i} < z_{a^*,j}}\left\{ \hat{z}_s(a,i) > z\right\}  \right)\\
    =& \prob\left(  \hat{z}_s(a^*,j) < z\right) + \prod_{i: z_{a,i} < z_{a^*,j}}\prob \left( \hat{z}_s(a,i) > z \right),
\end{align*}
where the second equality uses the independence among metrics. 
For an arbitrary constant $z$, we can bound both probabilities through Hoeffding's inequality as:
\begin{align*}
    \prob\left(  \hat{z}_s(a^*,j) < z\right) \leq \exp\left( -\frac{[z_{a^*,j} - z]_+^2}{2 \mathrm{Var}(\hat{z}_s(a^*,j))} \right),
\end{align*}
where the variance of the $z$ value estimate can be computed as:
\begin{align*}
    \mathrm{Var}(\hat{z}_s(a^*,j))
    =& \frac{\mathrm{Var}(\hat{\mu}_s(a^*,j)) + \mathrm{Var}(\hat{\mu}_s(0,j))}{\sigma_{a^*,j}^2 + \sigma_{0,j}^2}\\
    =& \left( \sqrt{\sum_{a\in \mathcal{A}_s} \max_i\rho_{a,i}^2} + \max_{a \in \gA_s} \max_i \lambda_{a,i} \right) \frac{\log_2 A}{T} \\
    &\cdot\left( \frac{\rho_{a^*,j}^2}{\max_i\rho_{a^*,i}^2} \sqrt{\sum_{a\in \mathcal{A}_s} \max_i\rho_{a,i}^2} + \frac{\lambda_{a^*,j}^2}{\max_{a \in \gA_s} \max_i \lambda_{a,i}} \right)\\
    =& \left( \rho_{\gA_s,\Sigma} + \lambda_{\gA_s,\Sigma} \right) \frac{\log_2 A}{T}\left( \frac{\rho_{a^*,j}^2}{\max_i\rho_{a^*,i}^2} \rho_{\gA_s,\Sigma} + \frac{\lambda_{a^*,j}^2}{\max_{a \in \gA_s} \max_i \lambda^2_{a,i}} \lambda_{\gA_s,\Sigma}\right)\\
    = & \left( \rho_{\gA_s,\Sigma} + \lambda_{\gA_s,\Sigma} \right)^2 \frac{\log_2 A}{T} \kappa^2_{\gA_s, a^*,j}.
\end{align*}
So the probability is bounded as follows:
\begin{align*}
    \prob\left(  \hat{z}_s(a^*,j) < z\right) \leq \exp\left( -\frac{[z_{a^*,j} - z]_+^2}{2 \left( \rho_{\gA_s,\Sigma} + \lambda_{\gA_s,\Sigma} \right)^2 \kappa^2_{\gA_s, a^*,j}} \frac{T}{\log_2 A} \right),
\end{align*}
Similarly, for the probability regarding treatment $a$, we could also bound them as:
\begin{align*}
    \prob\left(  \hat{z}_s(a,i) > z\right) \leq \exp\left( -\frac{[z- z_{a,i}]_+^2}{2 \left( \rho_{\gA_s,\Sigma} + \lambda_{\gA_s,\Sigma} \right)^2 \kappa^2_{\gA_s, a,i}} \frac{T}{\log_2 A} \right).
\end{align*}
Therefore, the intersection of events is bounded using independence as follows:
\begin{align*}
    &\prob\left( \bigcap_{i: z_{a,i} < z_{a^*,j}} \left\{ \hat{z}_s(a,i) > \hat{z}_s(a^*,j)\right\} \right)\\
    \leq& \exp\left( -\frac{T}{2 \left( \rho_{\gA_s,\Sigma} + \lambda_{\gA_s,\Sigma} \right)^2\log_2 A} \frac{[z_{a^*,j} - z]_+^2}{\kappa^2_{\gA_s, a^*,j}} \right)\\
    &+ \exp\left( -\frac{T}{2 \left( \rho_{\gA_s,\Sigma} + \lambda_{\gA_s,\Sigma} \right)^2\log_2 A} \sum_{i:z_{a,i}< z_{a^*,j}}\frac{[z - z_{a,i}]_+^2}{ \kappa^2_{\gA_s, a,i}}  \right)\\
    \leq& 2 \exp\left( -\frac{T}{2 \left( \rho_{\gA_s,\Sigma} + \lambda_{\gA_s,\Sigma} \right)^2\log_2 A} \min\left\{\frac{[z_{a^*,j} - z]_+^2}{\kappa^2_{\gA_s, a^*,j}}, \sum_{i:z_{a,i}< z_{a^*,j}}\frac{[z - z_{a,i}]_+^2}{ \kappa^2_{\gA_s, a,i}} \right\}  \right).
\end{align*}
Since $z$ is arbitrary chosen satisfying $z\leq z_{a^*, j}$. Therefore, choosing $z$ to minimize the exponent, we have:
\begin{align*}
    &\prob\left( \bigcap_{i: z_{a,i} < z_{a^*,j}} \left\{ \hat{z}_s(a,i) > \hat{z}_s(a^*,j)\right\} \right) \\
    \leq & 2 \exp\left( -\frac{T}{2 \left( \rho_{\gA_s,\Sigma} + \lambda_{\gA_s,\Sigma} \right)^2\log_2 A} \max_{z\leq z_{a^*,j}} \min\left\{\frac{[z_{a^*,j} - z]_+^2}{\kappa^2_{\gA_s, a^*,j}}, \sum_{i:z_{a,i}< z_{a^*,j}}\frac{[z - z_{a,i}]_+^2}{ \kappa^2_{\gA_s, a,i}} \right\}  \right).
\end{align*}
Finally, from a union bound, we will be able to bound the probability that a sub-optimal treatment has a better minimum z-value estimate than the best treatment as follows:
\begin{align*}
    &\prob\Big(\min_i \hat{z}_s(a,i) > \min_j \hat{z}_s(a^*,j)\Big)\\
    \leq& \sum_{j\in[M]} \prob\left( \min_i \hat{z}_s(a,i) > \hat{z}_s(a^*,j) \right)\\
    \leq &  2\sum_{j\in[M]} \exp\Bigg( -\frac{T}{2 \left( \rho_{\gA_s,\Sigma} + \lambda_{\gA_s,\Sigma} \right)^2\log_2 A}
    \max_{z\leq z_{a^*,j}} \min\Bigg\{\frac{[z_{a^*,j} - z]_+^2}{\kappa^2_{\gA_s, a^*,j}}, \sum_{i:z_{a,i}< z_{a^*,j}}\frac{[z - z_{a,i}]_+^2}{ \kappa^2_{\gA_s, a,i}} \Bigg\}  \Bigg)\\
    \leq & 2 M \exp\Bigg( -\frac{T}{2 \left( \rho_{\gA_s,\Sigma} + \lambda_{\gA_s,\Sigma} \right)^2\log_2 A} 
    \min_{j\in [M]} \max_{z\leq z_{a^*,j}} \min\Bigg\{\frac{[z_{a^*,j} - z]_+^2}{\kappa^2_{\gA_s, a^*,j}}, \sum_{i:z_{a,i}< z_{a^*,j}}\frac{[z - z_{a,i}]_+^2}{ \kappa^2_{\gA_s, a,i}} \Bigg\}  \Bigg).
\end{align*}
This concludes the proof.
\end{proof}
Before we proceed, we first take a deeper look at the exponent of the probability bound, specifically, the instance-dependent hardness constant, which measures how difficult to distinguish a sub-optimal treatment from the optimal treatment as follows:
\begin{align*}
    \min_{j\in [M]} \max_{z\leq z_{a^*,j}} \min\left\{\frac{[z_{a^*,j} - z]_+^2}{\kappa^2_{\gA_s, a^*,j}}, \sum_{i:z_{a,i}< z_{a^*,j}}\frac{[z - z_{a,i}]_+^2}{ \kappa^2_{\gA_s, a,i}} \right\}.
\end{align*}
We can replace the summation with any reward metric $i$, which results in the following lower bound:
\begin{align*}
    \min_{j\in [M]} \max_{i\in [M]} \max_{z\leq z_{a^*,j}}\min\left\{\frac{[z_{a^*,j} - z]_+^2}{\kappa^2_{\gA_s, a^*,j}}, \frac{[z - z_{a,i}]_+^2}{ \kappa^2_{\gA_s, a,i}} \right\} = \min_{j\in [M]} \max_{i\in [M]} \frac{[z_{a^*,j} - z_{a,i}]_+^2}{(\kappa_{\gA_s, a,i} + \kappa_{\gA_s, a^*,j})^2}.
\end{align*}
For simplicity of notation, we use the complexity notation as follows:
\begin{align*}
    D_{\gA_s, a} = \min_{i\in [M]} D_{\gA_s, a, i} = \min_{i\in [M]} \max_{j\in [M]} \frac{[z_{a^*,i} - z_{a,j}]_+^2}{(\kappa_{\gA_s, a,j} + \kappa_{\gA_s, a^*,i})^2}.
\end{align*}

Then, we bound the probability that the algorithm makes a mistake and excludes the best arm on stage $s$, summarized in the following lemma. 

\begin{lemma}
    Fix a stage $s$ and suppose the best treatment $a^*$ has not been eliminated in previous stages, the probability that it is eliminated at this stage is bounded as follows:
    \begin{align*}
        \prob\left(a^* \notin \gA_{s+1}\right)  \leq 6 M \exp\left( -\frac{T}{2\log_2 A} \min_{\gS \subset \gA: a^*\in \gS, |\gS| =\frac{|\gA|}{2^{s-1}} } \frac{\min_{a\in\gS'_{\ervc}}D^2_{\gS, a}}{ \left( \rho_{\gS,\Sigma} + \lambda_{\gS,\Sigma} \right)^2}\right),
    \end{align*}
    where $\gS$ is an arbitrary subset of all treatments that contains the best arm with a fixed cardinality, and $\gS'_{\ervc}$ is the subset of $\gS$ defined as follows:
    \begin{itemize}
        \item When $|\gS|\geq 4$, $\gS'_{\ervc}$ is the subset excluding at most $\frac{1}{4}|\gS|$ treatments with minimum $D_{\gS, a}$.
        \item When $|\gS|\leq 3$, $\gS'_{\ervc}$ is the subset of $\gS$ containing all sub-optimal treatments.
    \end{itemize}
\end{lemma}
\begin{proof}
Fix an arbitrary $\gA_s$ which contains the best treatment $a^*$. Suppose it has at least $4$ treatments. Let $\gA_s'$ as defined be the set of treatments in $\gA_s$ excluding at most $\frac{1}{4}|\gA_s|$ treatments with minimum $D_{\gA_s, a}$. If the best treatment is eliminated on round $s$, it must be the case that at least half of the treatments in $\gA_s$ have an empirical minimum $z$ value larger than the best treatment. In particular, at least $\frac{1}{3}$ of the treatments in $\gA_s'$ will have a larger empirical minimum $z$ value. So we have:
\begin{align*}
    \E\left[\sum_{a\in \gA_s'} \mathbbm{1}_{\min_i \hat{z}_s(a,i) > \min_j \hat{z}_s(a^*,j)} \right] 
    = &\sum_{a\in \gA_s'} \prob\left( \min_i \hat{z}_s(a,i) > \min_j \hat{z}_s(a^*,j) \right)\\
    \leq & 2 M \sum_{a\in \gA_s'} \exp\left( -\frac{T}{2\log_2 A} \frac{D^2_{\gA_s, a}}{ \left( \rho_{\gA_s,\Sigma} + \lambda_{\gA_s,\Sigma} \right)^2}\right)\\
    \leq & 2 M|\gA_s'| \exp\left( -\frac{T}{2\log_2 A} \frac{\min_{a\in\gA_s'}D^2_{\gA_s, a}}{ \left( \rho_{\gA_s,\Sigma} + \lambda_{\gA_s,\Sigma} \right)^2}\right).
\end{align*}
Therefore, by Markov's inequality, we will be able to establish an upper bound for the event that the best treatment is eliminated:
\begin{align*}
    \prob\left( a^* \notin \gA_{s+1} \right) 
    =& \prob\left( \sum_{a\in \gA_s'} \mathbbm{1}_{\min_i \hat{z}_s(a,i) > \min_j \hat{z}_s(a^*,j)} \geq \frac{|\gA_s'|}{3} \right)\\
    \leq & \frac{3}{|\gA_s'|} \E\left[\sum_{a\in \gA_s'} \mathbbm{1}_{\min_i \hat{z}_s(a,i) > \min_j \hat{z}_s(a^*,j)} \right]\\
    \leq & 6 M \exp\left( -\frac{T}{2\log_2 A} \frac{\min_{a\in\gA_s'}D^2_{\gA_s, a}}{ \left( \rho_{\gA_s,\Sigma} + \lambda_{\gA_s,\Sigma} \right)^2}\right).
\end{align*}
However, this bound is valid for a fixed $\gA_s$ which contains the best treatment, and is not valid if $\gA_s$ is a stochastic set. Therefore, we take a maximum over the probability upper bound shown above to make it compatible with stochastic sets as follows:
\begin{align*}
    \prob\left( a^* \notin \gA_{s+1} \right)  \leq 6 M \exp\left( -\frac{T}{2\log_2 A} \min_{\gS: a^*\in \gS, |\gS| =\frac{|\gA|}{2^{s-1}} } \frac{\min_{a\in\gS'_{\ervc}}D^2_{\gS, a}}{ \left( \rho_{\gS,\Sigma} + \lambda_{\gS,\Sigma} \right)^2}\right),
\end{align*}
where $\gS_{\ervc}'$ is defined in the statement of lemma.
\end{proof}
Finally, to prove Theorem~\ref{thm:estimate}, the best treatment $a^*$ will only be recommended at the end of the exploration phase if it survives all $\log_2 A$ stages. With a union bound, we have:
\begin{align*}
    \prob\left( \hat{a}\neq a^* \right) 
    \leq & \sum_{s=1}^{\log_2 A} \prob\left( a^* \text{ is eliminated at stage $s$} \right)\\
    \leq & 6 M \sum_{s=1}^{\log_2 A} \exp\left( -\frac{T}{2\log_2 A} \min_{\gS: a^*\in \gS, |\gS| =\frac{|\gA|}{2^{s-1}} } \frac{\min_{a\in\gS'_{\ervc}}D^2_{\gS, a}}{ \left( \rho_{\gS,\Sigma} + \lambda_{\gS,\Sigma} \right)^2}\right)\\
    \leq & 6 M \log_2 A \cdot \exp\left( -\frac{T}{2\log_2 A} \min_{\gS: a^*\in \gS} \frac{\min_{a\in\gS'_{\ervc}}D^2_{\gS, a}}{ \left( \rho_{\gS,\Sigma} + \lambda_{\gS,\Sigma} \right)^2}\right).
\end{align*}
Replacing $\rho_{\gS, \Sigma}$ and $\lambda_{\gS, \Sigma}$ with their definitions respectively, we obtain Theorem~\ref{thm:estimate}.

\section{Proof of Theorem~\ref{thm:conf}}\label{sec:proof-conf}
The proof of the confidence-based elimination strategy is similar to the proof of Theorem~\ref{thm:estimate} except that a new upper bound for the probability that some sub-optimal treatment $a$ will have a larger confidence $\delta_s(a)$ than the best treatment needs to be established. So following the proof of Theorem~\ref{thm:estimate}, suppose we fix a stage $s$ and the active treatment set $\gA_s$, we analyze the probability that a treatment $a$ has a larger confidence $\delta_s(a)$ than the best treatment $a^*$ in the following lemma, i.e.,
\begin{lemma}
    Fix a stage $s$ and the active treatment set $\gA_s$, we have:
    \begin{align*}
        \prob\left( \delta_s(a)>\delta_s(a^*) \right) \leq 2 M \exp\left( - \frac{T}{8 \log_2 A \left( \rho_{\gA_s,\Sigma} + \lambda_{\gA_s,\Sigma} \right)^2 } \tilde{D}_{\gA_s, a}^2 \right),
    \end{align*}
    where $(\rho_{\gA_s, \Sigma}, \lambda_{\gA_s,\Sigma})$ measures the relative variance of the active treatment set $\gA_s$, i.e.,
    \begin{align*}
        \rho_{\gA_s,\Sigma}^2 =  \sum_{a\in \mathcal{A}_s} \max_i\rho_{a,i}^2, \quad \lambda_{\gA_s,\Sigma}^2 = \max_{a \in \gA_s} \max_i \lambda^2_{a,i};
    \end{align*}
    and $\kappa^2_{\gA_s, a,i}$ measures the heterogeneity of relative variance of arm $a$ and metric $i$, i.e.,
    \begin{align*}
        \kappa^2_{\gA_s, a,i} = \frac{\rho_{a,i}^2}{\max_i\rho_{a,i}^2} \frac{\rho_{\gA_s, \Sigma} }{\rho_{\gA_s, \Sigma} + \lambda_{\gA_s,\Sigma}} + \frac{\lambda_{a,i}^2}{\max_{a \in \gA_s} \max_i \lambda^2_{a,i} }\frac{\lambda_{\gA_s,\Sigma}}{\rho_{\gA_s, \Sigma} + \lambda_{\gA_s,\Sigma}}.
    \end{align*}
    The instance-dependent complexity $\tilde{D}_{\gA_s, a}^2$ measures the hardness of identifying the best treatment $a^*$ against the sub-optimal treatment $a$, which is defined as: if $\kappa_{\gA_s, a, j} > \kappa_{\gA_s, a^*, i}$,
    \begin{align*}
        \tilde{D}^2_{\gA_s, a, i, j} =\min\left\{ \frac{[z_{a^*,i} - z_{a,j}]_+^2}{ (\kappa_{\gA_s, a, j} + \kappa_{\gA_s, a^*, i})^2 } ,  \frac{[z_{a^*,i} - z_{a,j}]_+^2}{\left( \kappa_{\gA_s, a, j} - \kappa_{\gA_s, a^*, i} \right)^2}+ \Delta_{\min}^2 - \frac{8A\log_2^2 A}{T}\right\},
    \end{align*}
    And if $\kappa_{\gA_s, a, j} \leq \kappa_{\gA_s, a^*, i}$,
    \begin{align*}
        \tilde{D}^2_{\gA_s, a, i, j} = \frac{[z_{a^*,i} - z_{a,j}]_+^2}{ (\kappa_{\gA_s, a,j} + \kappa_{\gA_s, a^*,i})^2 }.
    \end{align*}
    And we define:
    \begin{align*}
        \tilde{D}^2_{\gA_s, a} =& \min_i \tilde{D}^2_{\gA_s, a, i} = \min_{i \in [M]} \max_{j\in[M]} \tilde{D}^2_{\gA_s, a, i, j}.
    \end{align*}
\end{lemma}
\begin{proof}
Recall the definition of confidence $\delta_s(a)$, if $\delta_s(a)>\delta_s(a^*)$, it must be the case where
\begin{align*}
   \mathrm{UCB}^{\delta_s(a^*)}(a^*) = \max_{a'\in \gA_s} \mathrm{LCB}_s^{\delta_s(a^*)}(a') \leq \max_{a'\in \gA_s} \mathrm{LCB}_s^{\delta_s(a)}(a') = \mathrm{UCB}^{\delta_s(a)}(a) \leq \mathrm{UCB}^{\delta_s(a^*)}(a),
\end{align*}
where the inequality is because the lower confidence bound $\mathrm{LCB}_s^\delta(a)$ is monotonically non-decreasing concerning the confidence $\delta$. The second inequality is because the upper confidence bound $\mathrm{UCB}_s^\delta(a)$ is monotonically non-increasing with the confidence $\delta$. Then, this will induce the following event:
\begin{align*}
    \{\delta_s(a) > \delta_s(a^*)\} 
    \subset& \left\{ \min_i \left\{ \hat{z}_s(a,i) + b_s^{\delta_s(a^*)}(a,i) \right\} \geq \min_j \left\{ \hat{z}_s(a^*,j) + b_s^{\delta_s(a^*)}(a^*,j) \right\}   \right\}\\
    =& \bigcup_{j\in [M]} \left\{ \min_i \left\{ \hat{z}_s(a,i) + b_s^{\delta_s(a^*)}(a,i) \right\} \geq \hat{z}_s(a^*,j) + b_s^{\delta_s(a^*)}(a^*,j)    \right\}\\
    =& \bigcup_{j\in [M]} \bigcap_{i\in[M]} \underbrace{\left\{ \hat{z}_s(a,i) + b_s^{\delta_s(a^*)}(a,i) \geq \hat{z}_s(a^*,j) + b_s^{\delta_s(a^*)}(a^*,j)    \right\}}_{\gE_{j,i}}.
\end{align*}
Then, it suffices to analyze each event. We first take a deeper look at the confidence bonus with the relative-variance-based sampling strategy as follows:
\begin{align*}
    b_s^{\delta_s(a^*)}(a,i) 
    =& 2\sqrt{\left(\frac{\rho_{a,i}^2}{N_s(a)} + \frac{\lambda_{a,i}^2}{N_s(0)} \right) \log \left( \frac{|\gA_s| M}{\delta} \right)}.
\end{align*}
Specifically, we look at the variance terms as follows:
\begin{align*}
    \frac{\rho_{a,i}^2}{N_s(a)} + \frac{\lambda_{a,i}^2}{N_s(0)}
    =& \frac{\log_2 A}{T} \left( \sqrt{\sum_{a\in \mathcal{A}_s} \max_i\rho_{a,i}^2} + \max_{a \in \gA_s} \max_i \lambda_{a,i} \right)\\
    &\cdot\left(\frac{\rho_{a,i}^2}{\max_i\rho_{a,i}^2}\sqrt{\sum_{a\in \mathcal{A}_s} \max_i\rho_{a,i}^2}  + \frac{\lambda_{a,i}^2}{\max_{a \in \gA_s} \max_i \lambda_{a,i}} \right)\\
    =& \frac{\log_2 A}{T} \left( \rho_{\gA_s, \Sigma} + \lambda_{\gA_s,\Sigma} \right) \left(\frac{\rho_{a,i}^2}{\max_i\rho_{a,i}^2}\rho_{\gA_s, \Sigma}  + \frac{\lambda_{a,i}^2}{\max_{a \in \gA_s} \max_i \lambda^2_{a,i} }\lambda_{\gA_s,\Sigma}  \right)\\
    =& \frac{\log_2 A}{T} \left( \rho_{\gA_s, \Sigma} + \lambda_{\gA_s,\Sigma} \right)^2 \kappa^2_{\gA_s, a, i},
\end{align*}
where we use the definition:
\begin{align*}
    \kappa^2_{\gA_s, a, i} = \frac{\rho_{a,i}^2}{\max_i\rho_{a,i}^2} \frac{\rho_{\gA_s, \Sigma} }{\rho_{\gA_s, \Sigma} + \lambda_{\gA_s,\Sigma}} + \frac{\lambda_{a,i}^2}{\max_{a \in \gA_s} \max_i \lambda^2_{a,i} }\frac{\lambda_{\gA_s,\Sigma}}{\rho_{\gA_s, \Sigma} + \lambda_{\gA_s,\Sigma}}.
\end{align*}
Then, let $\hat{c}_s(a) = \sqrt{\log\left(\frac{|\gA_s| M}{\delta_s(a)}\right)}$ and $\Gamma^2_s = \frac{\log_2 A}{T} \left( \rho_{\gA_s, \Sigma} + \lambda_{\gA_s,\Sigma} \right)^2 $, each event $\gE_{j,i}$ becomes:
\begin{align*}
    \gE_{j,i} = \left\{ \hat{z}_s(a,i) + 2\kappa_{\gA_s, a, i} \Gamma_s \hat{c}_s(a^*) \geq \hat{z}_s(a^*,j) + 2\kappa_{\gA_s, a^*, j} \Gamma_s \hat{c}_s(a^*)    \right\}.
\end{align*}
We divide the metrics $i$ into two groups based on the relation between $\kappa_{\gA_s, a, i}$ and $\kappa_{\gA_s, a^*,j}$.

\textbf{Case 1.} $\kappa_{\gA_s, a, i} \leq \kappa_{\gA_s, a^*, j}$. For the first group, in order for event $\gE_{j,i}$ to hold, we must have:
\begin{align*}
    \gE_{j,i} \subset\left\{\hat{z}_s(a,i) \geq  \hat{z}_s(a^*,j) \right\}.
\end{align*}
By Hoeffding's inequality, we can bound the probability of this event as follows:
\begin{align*}
    \prob\left( \hat{z}_s(a,i) \geq  \hat{z}_s(a^*,j) \right) \leq \exp\left( - \frac{[z_{a^*,j} - z_{a,i}]_+^2}{2\left(\mathrm{Var}(\hat{z}_s(a,i)) + \mathrm{Var}(\hat{z}_s(a^*,j))\right)}  \right).
\end{align*}
The variance of estimate $\hat{z}_s(a^*,j)$ can be analyzed as follows:
\begin{align*}
    &\mathrm{Var}(\hat{z}_s(a^*,j))
    = \frac{\mathrm{Var}(\hat{\mu}_s(a^*,j)) + \mathrm{Var}(\hat{\mu}_s(0,j))}{\sigma_{a^*,j}^2 + \sigma_{0,j}^2}\\
    =& \left( \sqrt{\sum_{a\in \mathcal{A}_s} \max_i\rho_{a,i}^2} + \max_{a \in \gA_s} \max_i \lambda_{a,i} \right) \frac{\log_2 A}{T}
    \left( \frac{\rho_{a^*,j}^2}{\max_i\rho_{a^*,i}^2} \sqrt{\sum_{a\in \mathcal{A}_s} \max_i\rho_{a,i}^2} + \frac{\lambda_{a^*,j}^2}{\max_{a \in \gA_s} \max_i \lambda_{a,i}} \right)\\
    =& \left( \rho_{\gA_s,\Sigma} + \lambda_{\gA_s,\Sigma} \right) \frac{\log_2 A}{T}\left( \frac{\rho_{a^*,j}^2}{\max_i\rho_{a^*,i}^2} \rho_{\gA_s,\Sigma} + \frac{\lambda_{a^*,j}^2}{\max_{a \in \gA_s} \max_i \lambda^2_{a,i}} \lambda_{\gA_s,\Sigma}\right)\\
    = & \left( \rho_{\gA_s,\Sigma} + \lambda_{\gA_s,\Sigma} \right)^2 \frac{\log_2 A}{T} \kappa^2_{\gA_s, a^*,j}.
\end{align*}
Similarly, the variance of estimate $\hat{z}_s(a,i)$ can be analyzed as follows:
\begin{align*}
    \mathrm{Var}(\hat{z}_s(a,i))
    = \left( \rho_{\gA_s,\Sigma} + \lambda_{\gA_s,\Sigma} \right)^2 \frac{\log_2 A}{T} \kappa^2_{\gA_s, a,i}.
\end{align*}
Therefore, we have the following bound:
\begin{align*}
    \prob\left( \gE_{j,i} \right) \leq &\prob\left( \hat{z}_s(a,i) \geq  \hat{z}_s(a^*,j) \right)\\
    \leq& \exp\left( - \frac{T}{2\log_2 A \left( \rho_{\gA_s,\Sigma} + \lambda_{\gA_s,\Sigma} \right)^2} \frac{[z_{a^*,j} - z_{a,i}]_+^2}{ \kappa^2_{\gA_s, a,i} + \kappa^2_{\gA_s, a^*,j} }  \right)\\
    \leq & \exp\left( - \frac{T}{2\log_2 A \left( \rho_{\gA_s,\Sigma} + \lambda_{\gA_s,\Sigma} \right)^2} \frac{[z_{a^*,j} - z_{a,i}]_+^2}{ (\kappa_{\gA_s, a,i} + \kappa_{\gA_s, a^*,j})^2 }  \right)\\
    =& \exp\left( - \frac{T }{2\log_2 A \left( \rho_{\gA_s,\Sigma} + \lambda_{\gA_s,\Sigma} \right)^2} D^2_{\gA_s, a, j, i}  \right).
\end{align*}

\textbf{Case 2.} $\kappa_{\gA_s, a, i}  > \kappa_{\gA_s, a^*, j}$. For the second group, if $z_{a^*,j} - z_{a,i}<0$, we can simply bound the probability of this event by $1$. Assume $z_{a^*, j} > z_{a,i}$, and we can bound the event $\gE_{j,i}$ as:
\begin{align*}
    \gE_{j,i} \subset \underbrace{\left\{ \hat{z}_s(a,i) -\hat{z}_s(a^*,j) \geq - \frac{z_{a^*,j} - z_{a,i}}{2} \right\}}_{\gE_{j,i,1}} \cup \underbrace{\left\{ \hat{c}_s(a^*) \geq \frac{z_{a^*,j} - z_{a,i}}{4\left( \kappa_{\gA_s, a, i} - \kappa_{\gA_s, a^*, j} \right) \Gamma_s} \right\}}_{\gE_{j,i,2}}.
\end{align*}
The probability of the first event $\gE_{j,i,1}$ can be bounded with Hoeffding's inequality, similar to Case 1, as follows:
\begin{align*}
    \prob\left( \hat{z}_s(a,i) -\hat{z}_s(a^*,j) \geq - \frac{z_{a^*,j} - z_{a,i}}{2} \right)
    \leq& \exp\left( - \frac{\left(\frac{z_{a^*,j} - z_{a,i}}{2}\right)^2}{2\left(\mathrm{Var}(\hat{z}_s(a,i)) + \mathrm{Var}(\hat{z}_s(a^*,j))\right)}  \right)\\
    \leq & \exp\left( - \frac{T}{8\log_2 A \left( \rho_{\gA_s,\Sigma} + \lambda_{\gA_s,\Sigma} \right)^2} \frac{[z_{a^*,j} - z_{a,i}]_+^2}{ (\kappa_{\gA_s, a,i} + \kappa_{\gA_s, a^*,j})^2 }  \right)\\
    =& \exp\left( - \frac{T}{8\log_2 A \left( \rho_{\gA_s,\Sigma} + \lambda_{\gA_s,\Sigma} \right)^2} D_{\gA_s, a, j,i}^2  \right) .
\end{align*}
The second event $\gE_{j,i,2}$ induces a lower bound on the confidence to eliminate the best treatment $a^*$. Therefore, there must exist another arm $a'$, whose lower confidence bound is larger than the upper confidence bound of $a^*$ when the confidence bonus $b_s^{\delta}(a, i)$ is constructed with $\delta = \delta_s(a^*)$, i.e., under $\gE_{j,i,2}$, we have:
\begin{align*}
    \left\{ \exists a'\in\gA_s, \min_i \left\{ \hat{z}_s(a',i) - b_s^{\delta_s(a^*)}(a',i) \right\} \geq \min_{j'} \left\{ \hat{z}_s(a^*,j') + b_s^{\delta_s(a^*)}(a^*,j') \right\}  \right\}.
\end{align*}
Notice that in event $\gE_{j,i,2}$, reward metric $j$ is the metric which achieves the minimum in the right-hand side, so we can further rewrite this event as:
\begin{align*}
    \left\{ \exists a'\in\gA_s, \forall i\in [M], \hat{z}_s(a',i) - b_s^{\delta_s(a^*)}(a',i) \geq \hat{z}_s(a^*,j) + b_s^{\delta_s(a^*)}(a^*,j)  \right\}.
\end{align*}
Similarly, we start from the event for a fixed $a'$, a fixed metric $i'$ and a fixed $j$, i.e.,
\begin{align*}
    \gE_{j,i'}' = &\left\{  \hat{z}_s(a',i') - b_s^{\delta_s(a^*)}(a',i') \geq \hat{z}_s(a^*,j) + b_s^{\delta_s(a^*)}(a^*,j)  \right\}\\
    =&  \left\{ \hat{z}_s(a',i') - 2\kappa_{\gA_s, a', i'} \Gamma_s \hat{c}_s(a^*) \geq \hat{z}_s(a^*,j) + 2\kappa_{\gA_s, a^*, j} \Gamma_s \hat{c}_s(a^*) \right\} \\
    =& \left\{ \hat{z}_s(a',i') - \hat{z}_s(a^*,j) \geq 2\left(\kappa_{\gA_s, a', i'} + \kappa_{\gA_s, a^*, j}\right) \Gamma_s \hat{c}_s(a^*) \right\}.
\end{align*}
Then, under event $\gE_{j,i,2}$, we can lower bound the additional confidence based bonus term $\hat{c}_s(a^*)$ and obtain:
\begin{align*}
    \gE_{j,i'}' \subset \left\{ \hat{z}_s(a',i') - \hat{z}_s(a^*,j) \geq \frac{\kappa_{\gA_s, a', i'} + \kappa_{\gA_s, a^*, j}}{2\left(\kappa_{\gA_s, a, i} - \kappa_{\gA_s, a^*, j}\right)} (z_{a^*,j} - z_{a,i}) \right\}.
\end{align*}
From Hoeffding's inequality, the probability of this event can be bounded by:
\begin{align*}
    \prob\left( \gE_{j,i}' \right) \leq \exp\left( - \frac{\left(\frac{\kappa_{\gA_s, a', i'} + \kappa_{\gA_s, a^*, j}}{2\left(\kappa_{\gA_s, a, i} - \kappa_{\gA_s, a^*, j}\right)} (z_{a^*,j} - z_{a,i}) + [z_{a^*, j } - z_{a',i'}]_+ \right)^2}{2 \left(\mathrm{Var}(\hat{z}_s(a',i')) + \mathrm{Var}(\hat{z}_s(a^*,j))\right)} \right),
\end{align*}
where, following similar arguments from previous analysis, we can derive:
\begin{align*}
    \mathrm{Var}(\hat{z}_s(a',i')) + \mathrm{Var}(\hat{z}_s(a^*,j)) =& \left( \rho_{\gA_s,\Sigma} + \lambda_{\gA_s,\Sigma} \right)^2 \frac{\log_2 A}{T} \left(\kappa^2_{\gA_s, a^*,j} + \kappa^2_{\gA_s, a',i'} \right)\\
    \leq &\left( \rho_{\gA_s,\Sigma} + \lambda_{\gA_s,\Sigma} \right)^2 \frac{\log_2 A}{T} \left(\kappa_{\gA_s, a^*,j} + \kappa_{\gA_s, a',i'} \right)^2.
\end{align*}
So substitute this result in, we have:
\begin{align*}
    \prob\left( \gE_{j,i}' \right)
    \leq& \exp\left( - \frac{T}{8 \log_2 A \left( \rho_{\gA_s,\Sigma} + \lambda_{\gA_s,\Sigma} \right)^2 } \left(\frac{(z_{a^*,j} - z_{a,i})^2}{\left( \kappa_{\gA_s, a, i} - \kappa_{\gA_s, a^*, j} \right)^2} + \frac{[z_{a^*, j } - z_{a',i'}]_+^2}{\left(\kappa_{\gA_s, a^*,j} + \kappa_{\gA_s, a',i'} \right)^2}\right) \right)\\
    =& \exp\left( - \frac{T}{8 \log_2 A \left( \rho_{\gA_s,\Sigma} + \lambda_{\gA_s,\Sigma} \right)^2 } \left(\frac{(z_{a^*,j} - z_{a,i})^2}{\left( \kappa_{\gA_s, a, i} - \kappa_{\gA_s, a^*, j} \right)^2} + D_{\gA_s, a', j, i'}^2\right) \right).
\end{align*}
Finally, we use a union bound over $a'$ and choose the best $i'$ to minimize the probability upper bound as follows:
\begin{align*}
    \prob\left( \gE_{j,i,2} \right)
    \leq& \sum_{a' \neq a^*} \exp\left( - \frac{T}{8 \log_2 A \left( \rho_{\gA_s,\Sigma} + \lambda_{\gA_s,\Sigma} \right)^2 } \left(\frac{(z_{a^*,j} - z_{a,i})^2}{\left( \kappa_{\gA_s, a, i} - \kappa_{\gA_s, a^*, j} \right)^2} + \max_{i'} D_{\gA_s, a', j, i'}^2\right) \right)\\
    \leq & A \exp\left( - \frac{T}{8 \log_2 A \left( \rho_{\gA_s,\Sigma} + \lambda_{\gA_s,\Sigma} \right)^2 } \left(\frac{(z_{a^*,j} - z_{a,i})^2}{\left( \kappa_{\gA_s, a, i} - \kappa_{\gA_s, a^*, j} \right)^2} + \min_{a'\neq a^*} \max_{i'}D_{\gA_s, a', j, i'}^2\right) \right).
\end{align*}
Notice that the second term:
\begin{align*}
    \min_{a'\neq a^*} \max_{i'}D_{\gA_s, a', j, i'}^2
    \geq & \min_{a\neq a^*}\max_{i}[z_{a^*, j } - z_{a,i}]_+^2 \geq \min_{a\neq a^*}(z_{a^*, j } - z_{a,i^*})^2 \geq \underbrace{\min_{a\neq a^*}(z_{a^*, i^* } - z_{a,i^*})^2}_{\Delta_{\min}^2}.
\end{align*}
So the bound can be rewritten as follows:
\begin{align*}
    \prob\left( \gE_{j,i,2} \right)
    \leq& \exp\left( - \frac{T}{8 \log_2 A \left( \rho_{\gA_s,\Sigma} + \lambda_{\gA_s,\Sigma} \right)^2 } \left(\frac{(z_{a^*,j} - z_{a,i})^2}{\left( \kappa_{\gA_s, a, i} - \kappa_{\gA_s, a^*, j} \right)^2} + \Delta_{\min}^2 - \frac{8A\log_2^2 A}{T} \right) \right).
\end{align*}
Therefore, we can use a union bound to derive the probability of event $\gE_{j,i}$ as follows:
\begin{align*}
    \prob\left( \gE_{j,i} \right)
    \leq& \exp\left( - \frac{T}{8\log_2 A \left( \rho_{\gA_s,\Sigma} + \lambda_{\gA_s,\Sigma} \right)^2} \frac{[z_{a^*,j} - z_{a,i}]_+^2}{ (\kappa_{\gA_s, a,i} + \kappa_{\gA_s, a^*,j})^2 }  \right)\\
    &+ \exp\left( - \frac{T}{8 \log_2 A \left( \rho_{\gA_s,\Sigma} + \lambda_{\gA_s,\Sigma} \right)^2 } \left(\frac{[z_{a^*,j} - z_{a,i}]_+^2}{\left( \kappa_{\gA_s, a, i} - \kappa_{\gA_s, a^*, j} \right)^2}+ \Delta_{\min}^2 - \frac{8A\log_2^2 A}{T} \right) \right)\\
    \leq & 2\exp\left( - \frac{T}{8 \log_2 A \left( \rho_{\gA_s,\Sigma} + \lambda_{\gA_s,\Sigma} \right)^2 } \tilde{D}_{\gA_s, a, j,i}^2 \right),
\end{align*}
where we recall the definition that for $i$ where $\kappa_{\gA_s, a, i} > \kappa_{\gA_s, a^*, j}$:
\begin{align*}
    \tilde{D}_{\gA_s, a, j,i}^2 = \min\left\{ \frac{[z_{a^*,j} - z_{a,i}]_+^2}{ (\kappa_{\gA_s, a,i} + \kappa_{\gA_s, a^*,j})^2 } ,  \frac{[z_{a^*,j} - z_{a,i}]_+^2}{\left( \kappa_{\gA_s, a, i} - \kappa_{\gA_s, a^*, j} \right)^2}+ \Delta_{\min}^2 - \frac{8A\log_2^2 A}{T}\right\}.
\end{align*}
And for $i$ where $\kappa_{\gA_s, a, i} \leq \kappa_{\gA_s, a^*, j}$:
\begin{align*}
    \tilde{D}_{\gA_s, a, j,i}^2 = \frac{[z_{a^*,j} - z_{a,i}]_+^2}{ (\kappa_{\gA_s, a,i} + \kappa_{\gA_s, a^*,j})^2 }.
\end{align*}
So, with both cases analyzed, the intersection of events can be bounded with the event that minimizes the probability upper bound, i.e.,
\begin{align*}
    \prob\left( \bigcap_{i\in[M]} \gE_{j,i} \right)
    \leq& \min_{i\in [M]} \prob\left( \gE_{j,i} \right)\\
    = & \min\left\{ \min_{i:\kappa_{\gA_s, a, i} \leq \kappa_{\gA_s, a^*, j}} \prob\left( \gE_{j,i} \right), \min_{i:\kappa_{\gA_s, a, i} > \kappa_{\gA_s, a^*, j}} \prob\left( \gE_{j,i} \right) \right\}\\
    \leq & 2\exp\left( - \frac{T}{8 \log_2 A \left( \rho_{\gA_s,\Sigma} + \lambda_{\gA_s,\Sigma} \right)^2 } \max_{i\in[M] } \tilde{D}_{\gA_s, a, j,i}^2 \right).
\end{align*}
Then, the probability that a sub-optimal arm has better $\delta_s(a)$ than the optimal arm can be bounded with a union bound as follows:
\begin{align*}
    \prob\left( \delta_s(a)>\delta_s(a^*) \right) \leq& \sum_{j\in [M]} \prob\left( \bigcap_{i\in[M]} \gE_{j,i} \right)\\
    \leq & 2 \sum_{j\in [M]} \exp\left( - \frac{T}{8 \log_2 A \left( \rho_{\gA_s,\Sigma} + \lambda_{\gA_s,\Sigma} \right)^2 } \max_{i\in[M] } \tilde{D}_{\gA_s, a, j,i}^2 \right)\\
    \leq & 2 M \exp\left( - \frac{T}{8 \log_2 A \left( \rho_{\gA_s,\Sigma} + \lambda_{\gA_s,\Sigma} \right)^2 } \min_{j\in[M]}\max_{i\in[M] } \tilde{D}_{\gA_s, a, j,i}^2 \right)\\
    =& 2 M \exp\left( - \frac{T}{8 \log_2 A \left( \rho_{\gA_s,\Sigma} + \lambda_{\gA_s,\Sigma} \right)^2 } \tilde{D}_{\gA_s, a}^2 \right)
\end{align*}
This concludes the proof.
\end{proof}

Then, we bound the probability that the algorithm makes a mistake and excludes the best treatment on stage $s$, summarized in the following lemma. 

\begin{lemma}
    Fix a stage $s$ and suppose the best treatment $a^*$ has not been eliminated in previous stages, the probability that it is eliminated at this stage is bounded as follows:
    \begin{align*}
        \prob\left(a^* \notin \gA_{s+1}\right)  \leq 6 M \exp\left( -\frac{T}{2\log_2 A} \min_{\gS \subset \gA: a^*\in \gS, |\gS| =\frac{|\gA|}{2^{s-1}} } \frac{\min_{a\in\tilde{\gS}'_{\ervc}}\tilde{D}^2_{\gS, a}}{ \left( \rho_{\gS,\Sigma} + \lambda_{\gS,\Sigma} \right)^2}\right),
    \end{align*}
    where $\gS$ is an arbitrary subset of all treatments that contains the best arm with a fixed cardinality, and $\tilde{\gS}'_{\ervc}$ is the subset of $\gS$ defined as follows:
    \begin{itemize}
        \item When $|\gS|\geq 4$, $\tilde{\gS}'_{\ervc}$ is the subset excluding at most $\frac{1}{4}|\gS|$ treatments with minimum $\tilde{D}_{\gS, a}$.
        \item When $|\gS|\leq 3$, $\tilde{\gS}'_{\ervc}$ is the subset of $\gS$ containing all sub-optimal treatments.
    \end{itemize}
\end{lemma}
\begin{proof}
Fix an arbitrary $\gA_s$ which contains the best treatment $a^*$. Suppose it has at least $4$ treatments. Let $\gA_s'$ as defined be the set of treatments in $\gA_s$ excluding at most $\frac{1}{4}|\gA_s|$ treatments with minimum $\tilde{D}_{\gA_s, a}$. If the best treatment is eliminated on round $s$, it must be the case that at least half of the treatments in $\gA_s$ have a smaller confidence of mistake $\delta_s(a)$ larger than the best treatment. In particular, at least $\frac{1}{3}$ of the treatments in $\gA_s'$ will have a smaller confidence. So we have:
\begin{align*}
    \E\left[\sum_{a\in \gA_s'} \mathbbm{1}_{ \delta_s(a)>\delta_s(a^*)} \right] 
    = &\sum_{a\in \gA_s'} \prob\left(  \delta_s(a)>\delta_s(a^*) \right)\\
    \leq & 2 M \sum_{a\in \gA_s'} \exp\left( -\frac{T}{8\log_2 A} \frac{\tilde{D}^2_{\gA_s, a}}{ \left( \rho_{\gA_s,\Sigma} + \lambda_{\gA_s,\Sigma} \right)^2}\right)\\
    \leq & 2 M|\gA_s'| \exp\left( -\frac{T}{8\log_2 A} \frac{\min_{a\in\gA_s'}\tilde{D}^2_{\gA_s, a}}{ \left( \rho_{\gA_s,\Sigma} + \lambda_{\gA_s,\Sigma} \right)^2}\right).
\end{align*}
Therefore, by Markov's inequality, we will be able to establish an upper bound for the event that the best treatment is eliminated:
\begin{align*}
    \prob\left( a^* \notin \gA_{s+1} \right) 
    =& \prob\left( \sum_{a\in \gA_s'} \mathbbm{1}_{ \delta_s(a)>\delta_s(a^*)} \geq \frac{|\gA_s'|}{3} \right)
    \leq \frac{3}{|\gA_s'|} \E\left[\sum_{a\in \gA_s'} \mathbbm{1}_{ \delta_s(a)>\delta_s(a^*)} \right]\\
    \leq & 6 M \exp\left( -\frac{T}{8\log_2 A} \frac{\min_{a\in\gA_s'}\tilde{D}^2_{\gA_s, a}}{ \left( \rho_{\gA_s,\Sigma} + \lambda_{\gA_s,\Sigma} \right)^2}\right).
\end{align*}
However, this bound is valid for a fixed $\gA_s$ which contains the best treatment, and is not valid if $\gA_s$ is a stochastic set. Therefore, we take a maximum over the probability upper bound shown above to make it compatible with stochastic sets as follows:
\begin{align*}
    \prob\left( a^* \notin \gA_{s+1} \right)  \leq 6 M \exp\left( -\frac{T}{8\log_2 A} \min_{\gS: a^*\in \gS, |\gS| =\frac{|\gA|}{2^{s-1}} } \frac{\min_{a\in\tilde{\gS}'_{\ervc}}\tilde{D}^2_{\gS, a}}{ \left( \rho_{\gS,\Sigma} + \lambda_{\gS,\Sigma} \right)^2}\right),
\end{align*}
where $\tilde{\gS}_{\ervc}'$ is defined in the statement of lemma.
\end{proof}
Finally, to prove Theorem~\ref{thm:conf}, the best treatment $a^*$ will only be recommended at the end of the exploration phase if it survives all $\log_2 A$ stages. With a union bound, we have:
\begin{align*}
    \prob\left( \hat{a}\neq a^* \right) 
    \leq & \sum_{s=1}^{\log_2 A} \prob\left( a^* \text{ is eliminated at stage $s$} \right)\\
    \leq & 6 M \sum_{s=1}^{\log_2 A} \exp\left( -\frac{T}{8\log_2 A} \min_{\gS: a^*\in \gS, |\gS| =\frac{|\gA|}{2^{s-1}} } \frac{\min_{a\in\tilde{\gS}'_{\ervc}}\tilde{D}^2_{\gS, a}}{ \left( \rho_{\gS,\Sigma} + \lambda_{\gS,\Sigma} \right)^2}\right)\\
    \leq & 6 M \log_2 A \cdot \exp\left( -\frac{T}{8\log_2 A} \min_{\gS: a^*\in \gS} \frac{\min_{a\in\tilde{\gS}'_{\ervc}}\tilde{D}^2_{\gS, a}}{ \left( \rho_{\gS,\Sigma} + \lambda_{\gS,\Sigma} \right)^2}\right).
\end{align*}
Replacing $\rho_{\gS, \Sigma}$ and $\lambda_{\gS, \Sigma}$ with their definitions respectively, we obtain Theorem~\ref{thm:conf}.

\end{document}